\documentclass[a4paper, 11pt]{article}

\usepackage{geometry}
\usepackage{microtype}
\usepackage{graphicx}
\usepackage{booktabs} 
\usepackage[natbib=true, style=authoryear-comp, maxbibnames=14, maxcitenames=1, uniquelist=false, citetracker=true, backend=biber, sortcites=false, doi=false,,url=false, firstinits=true]{biblatex}
\DeclareNameAlias{sortname}{family-given}
\AtEveryCitekey{\ifciteseen{}{\defcounter{maxnames}{2}}}
\bibliography{quenched.bib}

\usepackage[colorlinks=True, citecolor=blue]{hyperref}

\usepackage{amsmath}
\usepackage{amssymb}
\usepackage{mathtools}
\usepackage{amsthm}

\usepackage{subfig}

\usepackage[capitalize,noabbrev]{cleveref}

\theoremstyle{plain}
\newtheorem{theorem}{Theorem}[section]
\newtheorem{proposition}[theorem]{Proposition}
\newtheorem{lemma}[theorem]{Lemma}

\theoremstyle{definition}
\newtheorem{definition}[theorem]{Definition}
\newtheorem{assumption}[]{Assumption}
\theoremstyle{remark}

\usepackage{authblk}
\usepackage[textsize=tiny]{todonotes}

\def\twofig{0.5\textwidth}
\def\threefig{0.3\textwidth}

\newcommand\revv[1]{{#1}}
\DeclareUnicodeCharacter{0308}{HERE!HERE!}

\author[1]{Martin Rouault\thanks{Corresponding author: \href{mailto: rouault.martin@protonmail.com}{rouault.martin@protonmail.com}}}
\author[1]{Rémi Bardenet}
\author[2]{Mylène Ma\"ida}
\affil[1]{\small Univ. Lille, CNRS, Centrale Lille, UMR 9189 -- CRIStAL, 59651 Villeneuve d'Ascq, France}
\affil[2]{\small Univ. Lille, CNRS, UMR 8524 -- Laboratoire Paul Painlevé, F-59000 Lille, France}

\begin{document}

\onecolumn
\title{Quenched large deviations for \\Monte Carlo integration with Coulomb gases}

\maketitle

\begin{abstract}
    Gibbs measures, such as Coulomb gases, are popular in modelling systems of interacting particles. 
Recently, we proposed to use Gibbs measures as randomized numerical integration algorithms with respect to a target measure $\pi$ on $\mathbb R^d$, following the heuristics that repulsiveness between particles should help reduce integration errors. 
A major issue in this approach is to tune the interaction kernel and confining potential of the Gibbs measure, so that the equilibrium measure of the system is the target distribution $\pi$. 
Doing so usually requires another Monte Carlo approximation of the \emph{potential}, i.e. the integral of the interaction kernel with respect to $\pi$. 
Using the methodology of large deviations from Garcia--Zelada (2019), we show that a random approximation of the potential preserves the fast large deviation principle that guarantees the proposed integration algorithm to outperform independent or Markov quadratures.
For non-singular interaction kernels, we make minimal assumptions on this random approximation, which can be the result of a computationally cheap Monte Carlo preprocessing. 
For the Coulomb interaction kernel, we need the approximation to be based on another Gibbs measure, and we prove in passing a control on the uniform convergence of the approximation of the potential.

\end{abstract}



\section{Introduction}
\label{s:introduction}

Numerical integration, or \emph{quadrature}, is of routine use in applied mathematics, such as in Bayesian inference \citep{robert_bayesian_2007}.
For a target measure $\pi$ on $\mathbb{R}^d$, a quadrature is a measure $\mu$ supported on a finite number $n$ of atoms such that one can bound the worst-case integration error
\begin{equation}
    \label{e:worst_case_error}
    \mathsf{E}_\mathcal{F}(\mu) = \sup_{f\in\mathcal{F}} \left\vert \int f\mathrm{d}(\pi-\mu)\right\vert 
\end{equation}
for a class of functions $\mathcal{F} \subset L^1(\pi)$.
A wealth of constructions have been investigated in moderate and large dimension. For instance, Monte Carlo methods correspond to taking $\mu$ to be a random measure; they come with strong probabilistic controls on \eqref{e:worst_case_error} when $\mathcal{F}$ is a singleton \citep{robert_monte_2004}.
On the other hand, quasi-Monte Carlo methods correspond to deterministic quadratures that come with uniform controls on a larger set $\mathcal{F}$ of smooth functions \citep{DiPi10}.

More recently, an intermediate string of works has been approximating $\pi$ by a minimizer of the so-called \emph{maximum mean discrepancy} (MMD)
\begin{align}
    \mu \mapsto I_{K}(\mu - \pi):= \iint K(x, y)\,\mathrm{d}(\mu -\pi)(x)\,\mathrm{d}(\mu - \pi)(y), 
    \label{e:mmd}
\end{align}
where $K$ is a suitable bivariate function called an \emph{interaction kernel}, and the minimization is carried over measures supported on $n$ points.
The motivation is twofold. 
First, when $K$ is bounded, one can show that \eqref{e:mmd} is the worst-case error \eqref{e:worst_case_error} where the set $\mathcal{F}$ of integrands is the unit ball of the reproducing kernel Hilbert space (RKHS) $\mathcal{H}_K$ of kernel $K$ \citep{sriperumbudur_hilbert_2009}, see \cref{p:sriperumbudur}.
Throughout the paper, we will abusively denote by $\mathrm{E}_{\mathcal{H}_K}$ the worst--case error in \cref{e:worst_case_error} when $\mathcal{F}$ is the unit ball of $\mathcal{H}_K$.
Second, as long as the \emph{kernel mean embedding} (also known as \emph{potential})
\begin{equation}
    \label{e:potential}
    U_K^\pi:z \mapsto \int K(z, x)\mathrm{d}\pi(x)
\end{equation}
can be numerically evaluated at any point of $\mathbb{R}^d$, \eqref{e:mmd} can be evaluated as well, which makes its minimization amenable to optimization algorithms. 
At the price of the evaluation of \eqref{e:potential}, MMD minimization seems to offer a compromise between fast scalable algorithms and small integration errors over an RKHS \citep{ghahramani_bayesian_2002, chen_super_samples_2012,bach_equivalence_2012,pronzato_bayesian_2020,mak_support_2018,chen_stein_2018,chen_stein_2020,korba_kernel_2021,dwivedi_kernel_2021,chatalic_efficient_2025}; see \citep[Table 1]{chatalic_efficient_2025} for an overview of existing methods.
    Very recently, \citet{chen_stationary_2025} have also shown that, while minimizing the MMD is a nonconvex problem, stationary points (instead of minimizers) of the MMD actually already lead to efficient quadrature, while being more computationally tractable.

Randomized quadratures based on determinantal point processes \citep{bardenet_monte_2020,MaCoAm20} have also been found to control the MMD \citep{BeBaCh19,BeBaCh20}.
In the same spirit, we proposed in \citep{rouault_monte_2024} a direct probabilistic relaxation of MMD minimization, where we draw the atoms $y_1, \dots, y_n$ of our quadrature from the (Gibbs) measure on $\left(\mathbb{R}^{d}\right)^{n}$ defined as
\begin{align}
\label{e:gibbs_measure}
    \mathrm{d}\mathbb{P}_{n, \beta_n}^{V}(y_1, \dots, y_n) = \frac{1}{Z_{n, \beta_n}^{V}}\exp\left(-\beta_n H_n(y_1, \dots, y_n)\right)\,\mathrm{d}y_1\,\dots\mathrm{d}y_n,
\end{align}
where $\beta_n>0$ is a parameter called the \emph{inverse temperature}, 
\begin{align}
\label{e:H_n}
    H_n(y_1, \dots, y_n) = \frac{1}{2 n^2}\sum_{i \neq j} K(y_i, y_j) + \frac{1}{n}\sum_{i = 1}^{n}V(y_i),
\end{align}
includes an interaction term that is an empirical proxy for \eqref{e:mmd}, $V:\mathbb{R}^d\mapsto \mathbb{R}$ is a \emph{potential} that goes to infinity as $\Vert y\Vert\rightarrow \infty$ sufficiently fast to \emph{confine} the nodes $y_1, \dots, y_n$, i.e. avoid the density in \eqref{e:gibbs_measure} to be maximal when the $y_i$s diverge to $\infty$, and $Z_{n, \beta_n}^{V}$ is a normalizing constant.

The family of distributions \eqref{e:gibbs_measure} encompasses well-known point processes. 
For instance, when
 \[ K(x,y) = g(x,y) := \left\{ \begin{array}{ll}
                                \frac{1}{\lvert x -y\rvert^{d-2}}, & \textrm{ if } d\ge 3,\\
                                -\log|x-y|, & \textrm{ if } d=2,
                               \end{array}
\right.\]
the Gibbs measure \eqref{e:gibbs_measure} is called a \emph{Coulomb gas}, or \emph{one-component plasma} \citep{serfaty_systems_2018,serfaty_lectures_2024}.
It models the electrostatic interaction of unit-charge elementary particles in $\mathbb{R}^{d}$ that are further confined by an external field $V$.
\revv{Alternatively}, when $d=1$ and $K(x,y) = -\log|x-y|$, the measure \eqref{e:gibbs_measure} is called \emph{one-dimensional log-gas}.
It appears naturally as the distribution of the eigenvalues of some classical ensembles of random matrices \citep{AnGuZe10}.

In general, the potential $V$ in \eqref{e:gibbs_measure} has to be confining enough to make the probability measure well-defined, i.e. to guarantee that $0 < Z_{n,\beta_n}^V<\infty.$ 
A particularly interesting choice is
\begin{align}
\label{e:jellium}
    V(z) = V^{\pi}(z) :=  - U_{K}^{\pi}(z) + \Phi(z), \quad z\in\mathbb{R}^d,
\end{align}
where $\Phi$ is a simple confining term that vanishes on the (known) support of $\pi$, and $U_{K}^{\pi}$ is the potential of $\pi$ with respect to $K$ given by~\eqref{e:potential}.
In that case, it is easy to check through the Euler-Lagrange equations that $\pi$ is the equilibrium measure of the system, in the sense that the empirical measure $\mu_n$ of the \emph{particles} $y_1, \dots, y_n$ converges to $\pi$ as $n$ goes to infinity, making $\mu_n$ a candidate quadrature rule in numerical integration w.r.t. $\pi$.
Additionally, for a smooth function $f$, the Monte Carlo estimator $\mu_n(f)$ is expected to have low mean-square error because the particles spread very regularly through $\mathbb{R}^d$, avoiding the quadrature to miss any relevant oscillation by leaving an uncovered hole in $\mathbb{R}^d$. 
Relatedly, in physics, when $K=g$ and $V=V^\pi$, a random configuration drawn from \eqref{e:gibbs_measure} is called \emph{jellium}, and one can imagine positive electric charges interacting with a negatively charged background distribution $\pi$.

%
%

One way to quantify the fast convergence to $\pi$ of the empirical measure of $y_1, \dots, y_n$ drawn from $\eqref{e:gibbs_measure}$ with $V=V^\pi$ in \eqref{e:jellium} is through large deviations principles (LDPs). 
In particular, \cite{chafai_first-order_2014} and \cite{garcia-zelada_large_2019} established that for a very general class of bounded interaction kernels $K$, the worst-case integration error on the unit ball of the RKHS $\mathcal{H}_K$ is small with overwhelming probability, i.e. for any $r > 0$ and $\epsilon > 0$, for large enough $n$,
\begin{align}
    \mathbb{P}_{n, \beta_n}^{V^{\pi}}\left(\mathsf{E}_{\mathcal{H}_K}\left(\frac{1}{n}\sum_{i = 1}^{n}\delta_{y_i}\right) > r\right) \leq \exp\left(- \left(\frac{1}{2}- \epsilon\right) \beta_n r^2\right),
    \label{e:ldp_gibbs_bounded}
\end{align} 
Similarly, taking for $K$ the (unbounded) Coulomb interaction \eqref{e:coulomb_kernel}, they prove
\begin{align}
    \mathbb{P}_{n, \beta_n}^{V^{\pi}}\left(\mathsf{E}_{\mathcal{BL}}\left(\frac{1}{n}\sum_{i = 1}^{n}\delta_{y_i}\right) > r\right) \leq \exp\left(- c \beta_n r^2\right),
    \label{e:ldp_gibbs}
\end{align}
for some constant $c>0$, where $\mathcal{BL}$ is the space of functions that are $1$-Lipschitz and bounded by $1$. 
Throughout the paper, when the pushforward of the Gibbs measure $\mathbb{P}_{n, \beta_n}^{V^{\pi}}$ by the application
$$
    y_1, \dots, y_n \mapsto \frac{1}{n}\sum_{i = 1}^{n}\delta_{y_i}
$$ 
satisfies a LDP, we will abusively say that the sequence of empirical measures $\xi_n = \frac{1}{n}\sum_{i = 1}^{n}\delta_{y_i}$ satisfies the same LDP.
The results~\eqref{e:ldp_gibbs_bounded} and \eqref{e:ldp_gibbs} are to compare to classical results for i.i.d. quadrature nodes, such as Sanov's theorem, where $\beta_n$ in the exponential is replaced by $n$. 
The same rate $n$ appears for Markov chain quadratures \citep{dembo_large_2010}.
As soon as $\beta_n$ grows faster than $n$, the right-hand sides of \eqref{e:ldp_gibbs_bounded} and \eqref{e:ldp_gibbs} thus decrease faster than i.i.d. or Markov chain quadratures, justifying our use of the terms \emph{overwhelming probability}.
This fast convergence motivated our study of numerical quadratures based on the Gibbs measure $\mathbb{P}_{n, \beta_n}^{V^{\pi}}$ in \citep{rouault_monte_2024}.

There are two main issues with this strategy, though.
First, we do not know of any exact sampling algorithm for the Gibbs measure \eqref{e:gibbs_measure}, so that one has to rely on approximate samplers like Markov chain Monte Carlo (MCMC) methods.
Second, and maybe more fundamentally, computing the potential \eqref{e:jellium} requires evaluating $U_{K}^{\pi}$, which requires integrating w.r.t. $\pi$, seemingly defeating our initial purpose.   
In this paper, we address this second issue, completely removing the need to evaluate $U_K^\pi$ while preserving the better-than-Monte-Carlo integration guarantees of Gibbs measures. 

Because the problem of evaluating $U_K^\pi$ is common to all algorithms that minimize the MMD \eqref{e:mmd}, workarounds have naturally already been investigated.
They come in two clusters.
First, one can rely on Stein's trick \citep{anastasiou_stein_2023}, which defines a kernel such that $U_{K}^{\pi} \equiv 0$, but modifies the set of functions $f$ that are efficiently integrated in an unclear way, see related work below.
\revv{Alternatively}, the majority of MMD minimization algorithms \citep{chen_parametric_2010,chen_super_samples_2012,bach_equivalence_2012,pronzato_bayesian_2020} builds a Monte Carlo approximation to \eqref{e:potential}, using a single draw of a random measure that approximates $\pi$, e.g. using MCMC.
If that random measure has a large number of atoms compared to the cardinality of the final quadrature $n$, this two-level Monte Carlo technique reduces the range of potential applications by forcing the user to spend a significant part of their computational budget on preparing the distribution from which they will draw their final quadrature. 
Apart from \cite{dwivedi_kernel_2021} and \cite{chatalic_efficient_2025}, who bypass the issue by controlling the difference in MMD of their $n$-point quadrature to an $n^2$-point MCMC reference, to our knowledge there is no guarantee on the worst-case integration error \eqref{e:worst_case_error} of any MMD minimization algorithm that includes this second level of approximation of \eqref{e:potential}.
Our contribution is to provide such a guarantee for Gibbs measures.

Formally, let $\nu_n$ be a random $n$-atom measure on $\mathbb{R}^d$, called the \emph{background}. 
Conditionally on $\nu_n$, let $y_1, \dots, y_n$ be drawn from the Gibbs measure \eqref{e:gibbs_measure} with confining potential 
\begin{equation}
    \label{e:random_potential}
    V^{\nu_n} = - U_K^{\nu_n} + \Phi,
\end{equation}
that is \eqref{e:potential} where we replace the target measure $\pi$ by the atomic measure $\nu_n$.
Following the vocabulary of large deviations theory \citep{garcia-zelada_large_2019}, we call the conditional distribution $\mathbb{P}_{n, \beta_n}^{\mathrm{Q}, \pi}$ of $y_1, \dots, y_n$ given $\nu_n$  a \emph{quenched} Gibbs measure. 
In this paper, we show large deviations principles for this quenched Gibbs measure.

Our first result deals with a bounded interaction kernel $K$ in the Gibbs measure.
We show in \cref{t:ldp_quenched_bounded} that, on any event such that $\nu_n$ converges to $\pi$, the distribution $\mathbb{P}_{n, \beta_n}^{\mathrm{Q}, \pi}$ of the quadrature $y_1, \dots, y_n$ almost surely satisfies the fast LDP \eqref{e:ldp_gibbs_bounded}, just like the Gibbs measure with intractable potential \eqref{e:potential}. 
Consequently, we avoid the evaluation of \eqref{e:potential} --the bottleneck of most MMD minimization methods-- at no statistical cost on the LDP.
Note that it is computationally easy to build a $\nu_n$ that converges almost surely to $\pi$ using standard Monte Carlo techniques, such as self-normalized importance sampling.

Our second result covers the more difficult case of the (singular) Coulomb kernel $K=g$ in \eqref{e:gibbs_measure}.
This time, we show that the quenched Gibbs measure almost surely preserves the LDP \eqref{e:ldp_gibbs} if the atoms of $\nu_n$ are themselves drawn from a Gibbs measure with Coulomb interaction, under weak requirements on the equilibrium measure, and with importance weights compensating for the mismatch between that equilibrium measure and the target $\pi$. 
Here again, we obtain the same statistical guarantee on the worst-case integration error without ever needing to evaluate the intractable potential \eqref{e:potential}. 
Seeing Gibbs measures as an avatar of MMD minimization, we believe that ours is the first work that formally includes the approximation of the kernel mean embedding $U_{K}^{\pi}$ in the analysis of its quadrature.
We hope that this can inspire the analysis of non-Gibbs-based MMD minimization algorithms.

As a side remark for the reader versed into the literature on the Coulomb interaction, we provide \emph{en route} a control on the uniform convergence of the approximation of the potential.
For a well-chosen regularization $K_\zeta$ of the Coulomb kernel, we show in \cref{p:max_pot} and \cref{p:max_pot_weight} that 
\[\underset{z \in \mathbb{R}^{d}}{\sup\,} \left\vert\nu_n(K_{\zeta}(z, \cdot)) - U_{g}^{\pi}(z)\right\vert \rightarrow 0\]   
almost surely as $n \rightarrow + \infty$.
We believe this intermediate result is of independent interest since the extrema of the potential field generated by a Coulomb gas have recently drawn attention, with exact scalings obtained at first order in dimension $2$ \citep{lambert_law_2024,peilen_maximum_2024}, more on this below.

The paper is organized as follows. In the end of this section, we quickly survey related work.
In \cref{s:definitions}, we rigorously introduce the objects of interest and our assumptions, and we state our main results.
Sections~\ref{s:proof_ldp} and \ref{s:proof_pot} are devoted to the proofs of the main results.
Numerical experiments, discussion on remaining practical issues and proofs of auxiliary lemmas are given in appendix.

\subsection*{Related work}
\label{s:introduction:related_work}

\subsubsection*{Large deviations}
Large deviations for Gibbs measures with pairwise interactions as in \cref{e:gibbs_measure} have been thoroughly studied by \citet{chafai_first-order_2014,garcia-zelada_large_2019}.
In our work, we need to further include a random environment, i.e. an empirical approximation to the potential.

Actually, \citep[Theorem 4.1 and Theorem 4.2]{garcia-zelada_large_2019} give a list of sufficient conditions to derive large deviations principles in varying environments, with applications to Gibbs measures.
Yet, the latter applications deal with Gibbs measures for which we cannot pre-specify the equilibrium measure, an important requirement for Monte Carlo integration where the equilibrium measure should be (at least an approximation of) the target integration measure.
Our proofs thus consist in checking the sufficient conditions of \cite{garcia-zelada_large_2019} for the potential \eqref{e:random_potential}.
Checking these conditions in our setting is however no sinecure, especially in the case of the Coulomb interaction kernel. 

\subsubsection*{Extrema of the potential of a Coulomb gas}
In dimension $1$ and $2$, the two-dimensional Coulomb interaction kernel $g(x, y) = -\log \lvert x - y\rvert$,  along with $V(x) = \lvert x \rvert^2 /2$ and $\beta_n=2n^2$, yields a Gibbs measure that also describes the distribution of the eigenvalues of random matrices, respectively the Gaussian unitary and Ginibre ensembles \citep{AnGuZe10}.
In this case, the interaction potential $\sum_{i = 1}^{n} g(z, x_i)$ generated at point $z$ by points $x_1, \dots, x_n$ drawn from \eqref{e:gibbs_measure} coincides with the log of the modulus of the characteristic polynomial of the random matrix, evaluated in $z$.
The behavior of the maximum of this quantity has been a topic of interest, following the famous conjectures of \citet{fyodorov_freezing_2012} in the one-dimensional case.
In dimension $2$, for general $V$ and $\beta_n = \beta n^2$, it has been for instance shown by \citet{lambert_maximum_2020,lambert_law_2024,peilen_maximum_2024} that
\begin{align}
    \underset{z \in D(x, r)}{\sup\,} \frac{n}{\log n}\left\{ \frac{1}{n}\sum_{i = 1}^{n}g(z, x_i)- \int g(z, x)\,\mathrm{d}\mu_V(x) \right\} \underset{n \rightarrow + \infty}{\rightarrow} \frac{1}{\sqrt{\beta}} \text{\, in probability,}\label{e:lln_max_pot}
\end{align}
where $\mu_V$ is the compactly supported measure that satisfies \cref{e:jellium} for a fixed $V$ and $D(x, r)$ is any disk strictly contained in the support of $\mu_V$.
Obtaining a similar behavior in dimension $d\geq3$ is still an open question up to our knowledge, but we show as a byproduct of our proof in \cref{p:max_pot} that for general $V$ and $\beta_n$ growing faster than $n$,
\[\underset{z \in \mathbb{R}^{d}}{\sup\,}\left\lvert \frac{1}{n}\sum_{i = 1}^{n} K_{\zeta}(z, x_i) - \int g(z, x)\,\mathrm{d}\mu_V(x)\right\rvert \underset{n \rightarrow + \infty}{\rightarrow} 0, \quad\text{$\mathbb{P}$-almost surely,}\]
where $K_{\zeta}$ is a regularized version of $g$.
Regularization is indeed needed to keep the supremum of the absolute value finite.
Obtaining the exact scaling for the leading order as in \cref{e:lln_max_pot} requires fine estimates on the fluctuations of the system at the second order, a hot topic in the field for $d \geq 3$.

\subsubsection*{Monte Carlo integration and KSDs}

Computing $U_{K}^{\pi}(z)$ has been up to our knowledge one of main limitations of MMD minimization algorithms.
One standard way to bypass this computation is to rely on a large MCMC chain targeting $\pi$ \citep{pronzato_bayesian_2020}, which is fair if the computational bottleneck is the number $n$ of evaluations of the test function $f$.
This can  be the case for instance in computational biology when each evaluation of the likelihood of the model can range from a few minutes in cardiac electrophysiology models \citep{johnstone_uncertainty_2016} to a few hours for population models in ecology \citep{purves_time_2013}.
Yet, this comes at the price of losing theoretical integration guarantees for Monte Carlo integration with respect to $\pi$.

Kernel Stein Discrepancies (KSDs, \citet{anastasiou_stein_2023}) have been introduced as a way to bypass this computation.
Indeed, given a kernel $K$ and a target measure $\pi$, one can define a new kernel $K_{\pi}$ with closed form such that $U_{K_{\pi}}^{\pi}(z) = 0$ for all $z$ under a mild regularity assumption on $K$ when $\pi$ vanishes.
The latter can always be ensured even when $\pi$ does not vanish up to multiplying the kernel $K$ by a vanishing weight function at the boundary of the integration domain \citep{oates_convergence_2017}.
The kernel Stein discrepancy is then defined as the square root of the MMD in \cref{e:mmd} with kernel $K_{\pi}$.
Minimizing this quantity has led to fast and practical sampling algorithms \citep{chen_stein_2018,chen_stein_2020,korba_kernel_2021} along with decay rates in $n^{-1/2}$ for the KSD with target $\pi$ in some cases.
Yet, replacing $K$ by $K_{\pi}$ replaces the set of test functions $\mathcal{H}_K$ by $\mathcal{H}_{K_\pi}$, and there is little knowledge about the kind of functions contained in $\mathcal{H}_{K_\pi}$.
There is evidence that the quadrature is influenced in that case by pathological test functions that oscillate wildly in regions where $\pi$ puts little mass \citep{benard_kernel_2023}.
Even though one is able to show \citep{gorham_measuring_2020,kanagawa_controlling_2024} that under some assumptions on $K$, $\mathcal{H}_{K_\pi}$ is still large enough to imply convergence in Wassertein or bounded Lipschitz distance, we only have access so far to estimates of the form
\[\mathsf{E}_{\mathcal{BL}}(\mu) \in \mathcal{O}\left(\left(\log \frac{1}{\mathrm{KSD}(\mu, \pi)}\right)^{-1/2}\right)\]
as the $\mathrm{KSD}$ decreases.
Thus, the decay rates in $\mathcal{H}_{K_{\pi}}$ only transfer as slow logarithmic decay rates for meaningful distances for Monte Carlo integration.

\section{Definitions and main results}
\label{s:definitions}

Let $d \geq 1$ and $\pi$ be a probability measure on $\mathbb{R}^{d}$, which we call the \emph{target} distribution.
Recall that our goal is to get a computationally tractable empirical measure supported on $n$ points that approximates $\pi$.

\begin{assumption}
\label{assum:pi}
    The support $S_{\pi}$ of $\pi$ is compact.
    In particular, there exists $R>0$ such that $S_{\pi} \subset B(0, R)$.
    Furthermore, we assume that $\pi$ has a continuous density $\pi'$ on the interior of $S_{\pi}$ and that it is bounded on $S_{\pi}$.
\end{assumption}

As explained in \cref{s:introduction}, we will build our quadrature by balancing two effects: a repulsive pairwise interaction preventing points to be too close from one another, and a confinement term favoring regions of large mass under $\pi$.
These two effects are modeled by quantities known as energies and interaction kernels, and are combined in a Gibbs measure.

\begin{definition}[Energies]
\label{d:energies}
Let $K : \mathbb{R}^{d} \times \mathbb{R}^{d} \rightarrow \mathbb{R}\cup\{\pm \infty\}$ be a symmetric (interaction) kernel.
For any finite signed Borel measures $\mu, \nu$ on $\mathbb{R}^{d}$, whenever the following quantities are well defined, we denote by
$$
    U_{K}^{\mu}(z) = \int K(z, x)\,\mathrm{d}\mu(x)
$$ 
the interaction potential at point $z \in \mathbb{R}^{d}$, and by
$$
    I_K(\mu, \nu) = \iint K(x, y)\,\mathrm{d}\mu(x)\,\mathrm{d}\nu(y)
$$ 
the interaction energy; we also set $I_{K}(\mu) = I_{K}(\mu, \mu)$ for convenience.
For any probability measure $\mu$ on $\mathbb{R}^{d}$ and any confinement $ V : \mathbb{R}^{d} \rightarrow \mathbb{R}$, we define
\begin{equation}
    I_{K}^{V}(\mu) = \frac{1}{2}\iint \left\{ K(x, y) + V(x) + V(y)\right\}\,\mathrm{d}\mu(x)\,\mathrm{d}\mu(y),
    \label{e:energy_with_external_potential}
\end{equation}
again whenever it is well defined.
\end{definition}

Throughout the paper, some particular kernels will play an important role.
\begin{definition}
    \label{d:coulomb_kernel}
        For $x, y \in \mathbb{R}^{d}$, when $d \geq 3$, the Coulomb interaction kernel is given by 
        \begin{align}
        \label{e:coulomb_kernel}
            g(x, y) = \frac{1}{\lvert x - y\rvert^{d-2}}.
        \end{align}
    \end{definition}

\begin{assumption}[Non-singular kernel]
\label{assum:non_singular}

Let $K : \mathbb{R}^{d} \times \mathbb{R}^{d} \rightarrow \mathbb{R}$ be a non-negative continuous, symmetric interaction kernel, such that $C:= {\sup}_x\, K(x, x) < +\infty$.
Furthermore, $K$ is integrally strictly positive definite; i.e. $I_{K}(\mu - \nu) > 0$ for any finite signed Borel measures $\mu$ and $\nu$ \revv{such that $\mu \neq \nu.$}
\end{assumption}

For a kernel $K$ satisfying \cref{assum:non_singular}, the interaction energy between two distributions coincides with the squared worst-case integration error for test functions in the reproducing kernel Hilbert space $\mathcal{H}_K$ of kernel $K$.

 \begin{proposition}[\cite{sriperumbudur_hilbert_2009}] \label{p:sriperumbudur}
    Let $K$ satisfy \cref{assum:non_singular}, then the reproducing kernel Hilbert space $(\mathcal{H}_{K}, \langle .\,, \,.\rangle_{\mathcal{H}_K})$ induced by $K$ is well-defined (see for instance \cite{pronzato_bayesian_2020}).
    The boundedness on the diagonal of $K$ implies that $0 \leq K(x, y) \leq C$ for all $x, y$\footnote{e.g. by Cauchy-Schwarz in $\mathcal{H}_{K}$.}, and we have the ``duality" formula
    \begin{align}
    \label{e:duality_formula}
        I_{K}(\mu - \nu) = \left(\underset{\lVert f \rVert_{\mathcal{H}_{K}} \leq 1}{\sup} \left\lvert \int f\,\mathrm{d}(\mu - \nu)\right\rvert\right)^2.
    \end{align}
 \end{proposition}

Non-singular kernels such as the Gaussian or Matern kernels are standard in kernel regression \citep{rasmussen_gaussian_2005}.
Other natural examples come from regularizing singular kernels. 
For instance, for $\epsilon > 0$ and $s > 0$, consider the regularized Riesz interaction
\begin{align}
    \label{e:riesz_kernel_reg}
    K_{s, \epsilon}(x, y) = \frac{1}{\left(\lvert x -y \rvert^2 + \epsilon^2\right)^{s/2}}.
\end{align}
In particular, for $n \geq 1$ and $\zeta > 0$, we will also use the following non-singular regularization of the Coulomb interaction $g$ in \eqref{e:coulomb_kernel},
\begin{align}
\label{e:coulomb_kernel_reg}
K_{\zeta}(x, y) = \frac{1}{\left(\lvert x -y \rvert^2 + n^{-2\zeta}\right)^{\frac{d-2}{2}}}.
\end{align}

We will work with the following assumptions on the confinement $V$.

\begin{assumption}
    \label{assum:confinement}
    Let $V : \mathbb{R}^{d} \mapsto \mathbb{R}$ be a continuous function such that
    $V(x) \rightarrow + \infty$ as $\lvert x \rvert \rightarrow +\infty$.
    Without loss of generality, we will always assume that $V \geq 0$. 
    Furthermore, we assume that there exists $c' > 0$ such that $\int \exp\left(-c' V(x)\right)\,\mathrm{d}x < +\infty$.
\end{assumption}

We quickly give basic facts concerning $V$ and its link with the minimizers of the energies in \cref{d:energies}, see for instance \citep{pronzato_bayesian_2020,chafai_first-order_2014} for references. 
Under \cref{assum:confinement},
the energies of \cref{d:energies} are well defined when $K$ either satisfies \cref{assum:non_singular} or is the Coulomb kernel \eqref{e:coulomb_kernel}, though the energy is possibly infinite in the Coulomb case.
In the same way, $I_{K}^{V}$ in \eqref{e:energy_with_external_potential} is known to have a unique minimizer over the set of probability distributions, called the equilibrium measure $\mu_V$.
Furthermore, $\mu_V$ then has finite energy and compact support.

Although the so-called Euler-Lagrange equations characterize the equilibrium measure for general interaction kernels, the computation of the equilibrium measure is not explicit in general. 
Yet there are exceptions.
First, when $K = g$, a few choices of potential $V$ do lead to explicit expressions. 
For instance, for $V(x) = \frac{d-2}{2 R^d}\lvert x \rvert^{2}$, $\mu_V$ is the uniform measure on $B(0, R)$.
This is the main reason why the Coulomb kernel $g$ will play an important role in the upcoming results, although it involves some additional theoretical difficulties.
Second, and of paramount important for our Monte Carlo ambitions, we can force the equilibrium measure to be a given distribution $\pi$.
Indeed, under \cref{assum:pi} and when $K$ either satisfies \cref{assum:non_singular} or is the Coulomb kernel \eqref{e:coulomb_kernel}, setting 
\begin{equation}
    V^{\pi}(z) := 
        \begin{cases}
                    - U_{K}^{\pi}(z) \text{ for } z \in B(0, R)\\
        - U_{K}^{\pi}(z) + \Phi(z) \text{ otherwise}
         \end{cases},
\label{e:euler_lagrange}
\end{equation}
with for instance $\Phi(z) = \left[\lvert z\rvert^{2} - R^2\right]_{+},$  yields that $V^{\pi}$ satisfies \cref{assum:confinement} and that $\mu_{V^{\pi}} = \pi.$
We could make a more general choice of $\Phi$ to confine the points in $B(0, R)$, but the quadratic choice is quite canonical.
Additionally, note that with $V=V^\pi$, the energy $I_{K}^{V^{\pi}}(\mu)$ in \cref{d:energies} writes as
\[I_{K}^{V^{\pi}}(\mu) = \frac{1}{2}I_{K}(\mu - \pi) - \frac{1}{2}I_{K}(\pi) + \int_{z \notin B(0, R)} \Phi(z)\,\mathrm{d}\mu(z)\]
when it is finite, so that minimizing this quantity is indeed natural if one wishes the distribution $\mu$ to be close to $\pi$.
The reader should keep in mind that evaluating $V^\pi$ in \eqref{e:euler_lagrange} requires evaluating $U_K^\pi$, which is itself an integral w.r.t. $\pi$; this \emph{a priori} forbids the use of $V^\pi$ in a Monte Carlo integration algorithm with target distribution $\pi$.

We now define the probability measure from which we want the points to be drawn from, promoting configurations of points with small energy $I_{K}^{V}$ up to diagonal terms.
\begin{definition}[Gibbs measure]
    \label{d:gibbs}
    Let $K = g$ or be a non-singular kernel $K$ satisfying \cref{assum:non_singular}.
    Let $V$ satisfy \cref{assum:confinement} and $\beta_n \geq c' n$ for all $n \in \mathbb{N}$.
    Let $n \geq 2$.
    For $x_1, \dots, x_n \in \mathbb{R}^{d}$, define 
    \[
        H_n(x_1, \dots, x_n) = \frac{1}{2 n^2}\sum_{i \neq j} K(x_i, x_j) + \frac{1}{n}\sum_{i = 1}^{n} V(x_i)
    \]
    and
    \[
        \mathrm{d}\mathbb{P}_{n, \beta_n}^{V}(x_1, \dots, x_n) = \frac{1}{Z_{n, \beta_n}^{V}}\exp\left(-\beta_n H_n(x_1, \dots, x_n)\right)\,\mathrm{d}x_1\,\dots\,\mathrm{d}x_n,
    \]
    where $Z_{n, \beta_n}^{V}$ is the partition function normalizing $\mathbb{P}_{n, \beta_n}^{V}$.
    When $x_1, \dots, x_n \sim \mathbb{P}_{n, \beta_n}^{V}$ and $\xi_n = \frac{1}{n} \sum_{i=1}^n \delta_{x_i}$ and when talking about the asymptotics of the sequence of random measures $(\xi_n)$, we will assume they are all defined on the same probability space $(\Omega, \mathcal{F}, \mathbb{P})$.
\end{definition}

We recall a few results from \cite{chafai_first-order_2014}.
Under the assumptions of \cref{d:gibbs}, $Z_{n, \beta_n}^{V}$ is positive and finite, making the Gibbs measure $\mathbb{P}_{n, \beta_n}^{V}$ a well-defined probability distribution.\footnote{
    When $K=g$, it has been called a \emph{Coulomb gas}, \emph{one-component plasma}, or \emph{jellium}.
}
Moreover, if $x_1, \dots, x_n \sim \mathbb{P}_{n, \beta_n}^{V}$ and $\xi_n = \frac{1}{n} \sum_{i=1}^n \delta_{x_i}$, it is known that in the low-temperature regime $\beta_n / n \longrightarrow + \infty$ which we denote by $\beta_n \gg n$, $\xi_n$ converges in law to $\mu_V$, $\mathbb{P}$-almost surely, as $n$ grows.
Actually, in the same low-temperature regime, $\xi_n$ satisfies a large deviations principle (LDP) with speed $\beta_n$ and rate function $\mu \mapsto I_{K}^{V}(\mu) - I_{K}^{V}(\mu_V)$, meaning in particular that for any $r > 0$, there exists $n_0$ such that for any $n \geq n_0$,
\begin{align}
    \mathbb{P}_{n, \beta_n}^{V}\left(\xi_n \notin B_{\mathbf{d}}(\mu_V, r)\right) \leq \exp\left(-\beta_n \underset{\mu \not\in B_{\mathbf{d}}(\mu_V, r)}{\inf\,} \{I_{K}^{V}(\mu) - I_{K}^{V}(\mu_V)\}\right)\label{e:bound_gibbs}
\end{align}
where $ B_{\mathbf{d}}(\mu_V, r) := \{ \nu \in \mathcal P(\mathbb R^d), \mathbf{d}(\mu_V, \nu) < r\}$, and $\mathbf{d}$ is any distance on probability distributions compatible with the convergence in distribution.
In particular this implies \cref{e:ldp_gibbs} and \cref{e:ldp_gibbs_bounded}.

Recall that our motivation of numerical integration w.r.t. $\pi$ hints at the choice $V = V^{\pi}$ as in \cref{e:euler_lagrange}, which however involves an intractable integral with respect to $\pi$.
We will replace this integral itself by an easier-to-draw Monte Carlo estimator $V_n$ of $V^\pi$, using a quadrature $\nu_n$.
Conditionally on $V_n$, we call the Gibbs measure with potential $V_n$ the \emph{quenched} Gibbs measure.

\begin{definition}[Quenched Gibbs measure]
\label{d:gibbs_quenched}
    Under Assumptions~\ref{assum:pi} and \ref{assum:confinement}, further assume that $\beta_n \gg n$, and let $n\geq 2$.
    Let $\nu_n$ be any random finitely supported measure on $\mathbb{R}^d$.
    We define a function $V_n: \mathbb{R}^d\rightarrow \mathbb{R}$ as follows.
    When $K$ satisfies \cref{assum:non_singular}, we let
    \begin{equation}
        V_n(z) := 
        \begin{cases}
                    - U_{K}^{\nu_n}(z) \text{ for } z \in B(0, R)\\
        - U_{K}^{\nu_n}(z) + \Phi(z) \text{ otherwise}
         \end{cases},
    \label{e:quenched_potential_bounded}
    \end{equation}
    where $\Phi(z) = [\lvert z \rvert^2 - R^2]_{+}$.
    Similarly, when $d\geq 3$ and $K = g$ in \cref{d:coulomb_kernel}, let $\zeta > 0$ and
    \begin{equation}
        V_n(z) := 
            \begin{cases}
                        - U_{K_{\zeta}}^{\nu_n}(z) \text{ for } z \in B(0, R)\\
            - U_{K_{\zeta}}^{\nu_n}(z) + \Phi(z) \text{ otherwise}
            \end{cases},
    \label{e:quenched_potential_coulomb}
    \end{equation}
    where $K_{\zeta}$ is the regularized interaction kernel defined in~\eqref{e:coulomb_kernel_reg}.
    We are now ready to define the quenched Gibbs measure as the Gibbs measure with random confinement $V_n$, i.e.
    \[\mathrm{d}\mathbb{P}_{n, \beta_n}^{\mathrm{Q}}(y_1, \dots, y_n ; \nu_n) = \mathrm{d}\mathbb{P}_{n, \beta_n}^{V_n}(y_1, \dots, y_n).\]
    We denote by $H_{n}^{\mathrm{Q}}$ the associated discrete energy i.e.
       \begin{equation}
        \label{e:discrete_energy}
      H_{n}^{\mathrm{Q}}(y_1, \dots, y_n) = \frac{1}{2n^2}\sum_{i \neq j} K(y_i, y_j) + \frac{1}{n}\sum_{i = 1}^{n} V_n(y_i),
       \end{equation}
    and by $Z_{n, \beta_n}^{V_n}$ the associated partition.
    If $y_1, \dots, y_n \sim \mathbb{P}_{n, \beta_n}^{\mathrm{Q}, \nu_n}$, we denote by $\mu_{n}^{\mathrm{Q}} = \frac{1}{n}\sum_{i = 1}^{n} \delta_{y_i}$ the quenched empirical measure.
\end{definition}

\begin{proposition}
    Under the assumptions of \cref{d:gibbs_quenched}, the quenched Gibbs measure $\mathbb{P}_{n, \beta_n}^{\mathrm{Q}}(\,\cdot\,; \nu_n(\omega))$ is well defined, for $\mathbb{P}$-almost all $\omega$.
\end{proposition}

\begin{proof}
    It is enough to show that $V_n$ satisfies \cref{assum:confinement}, $\mathbb{P}$-almost surely, since this implies that $Z_{n, \beta_n}^{V_n}$ is positive and finite $\mathbb{P}$-almost surely;  see \citep{chafai_first-order_2014}.
    In both cases of \cref{e:quenched_potential_bounded} and \cref{e:quenched_potential_coulomb}, $V_n$ is clearly continuous and satisfies $V_n(z) \rightarrow +\infty$ as $\lvert z \rvert \rightarrow +\infty$, since $K$ is either bounded or decays to $0$.
    Moreover, as the kernel mean embedding term $U_{K}^{\nu_n}$ is continuous and bounded in each case, we clearly have
    \[
        \int \exp\left(-c' V_n(z)\right)\,\mathrm{d}z  \leq C \int \exp\left(- c' \Phi(z)\right)\,\mathrm{d}z < +\infty
    \]
    for all $c' > 0$, so that \cref{assum:confinement} is satisfied.
\end{proof}

The atoms $x_1, \dots, x_n$ of $\nu_n$ in \cref{d:gibbs_quenched} act as a random environment (or \emph{background}) that approximates $\pi$.
Then, the new points $y_1, \dots, y_n \sim \mathbb{P}_{n, \beta_n}^{\mathrm{Q}, \pi}$ repel each other, while the random environment acts as a confinement, attracting the nodes $y_1, \dots y_n$ towards regions of high probability under the target distribution $\pi$.
Crucially, all of this happens without ever needing to evaluate the intractable potential $U_{K}^{\pi}$ in \eqref{e:euler_lagrange}.
Our goal is now to show that for \emph{$\mathbb{P}$-almost-all background realization}, the quenched empirical measure still converges to the target distribution $\pi$ at the same speed as \eqref{e:bound_gibbs}. 
Without further assumption, we can already state our main result for bounded kernels.

\begin{theorem}
\label{t:ldp_quenched_bounded}

Under \cref{assum:pi} and \cref{assum:confinement}, further assume $\beta_n \gg n$.
Let $K$ be a non-singular kernel satisfying \cref{assum:non_singular}.
Let $(\nu_n)_n$ be any family of discrete probability measures on $(\Omega, \mathcal{F}, \mathbb{P})$ such that $\nu_n \rightarrow \pi$, $\mathbb{P}$-almost surely.
Then for $\mathbb{P}$-almost all $\omega$, the quenched empirical measure $\mu_{n}^{\mathrm{Q}} = \frac{1}{n}\sum_{i = 1}^{n}\delta_{y_i}$, where $y_1, \dots, y_n \sim \mathbb{P}_{n, \beta_n}^{\mathrm{Q}}(.\, ; \nu_n(\omega))$, satisfies a large deviations principle as $n \rightarrow +\infty$, at speed $\beta_n$ and with good rate function $\mu \mapsto I_{K}^{V^{\pi}}(\mu) - I_{K}^{V^{\pi}}(\pi)$, whose minimizer is $\pi$.
More precisely, for $\mathbb{P}$-almost all $\omega$, for any Borel set $A \subset \mathcal{P}(\mathbb{R}^{d})$,

\begin{align}
   - \underset{\mu \in \overset{\circ}{A}}{\inf\,} \{I_{K}^{V^{\pi}}(\mu) - I_{K}^{V^{\pi}}(\pi)\} &\leq \underset{n \rightarrow +\infty}{\underline{\lim}}\,\frac{1}{\beta_n}\log\mathbb{P}_{n, \beta_n}^{\mathrm{Q}}\left(\mu_{n}^{\mathrm{Q}} \in A; \nu_n(\omega)\right) \nonumber\\
   &\revv{\leq} \underset{n \rightarrow +\infty}{\overline{\lim}}\,\frac{1}{\beta_n}\log\mathbb{P}_{n, \beta_n}^{\mathrm{Q}}\left(\mu_{n}^{\mathrm{Q}} \in A; \nu_n(\omega)\right) \leq - \underset{\mu \in \overline{A}}{\inf\,} \{I_{K}^{V^{\pi}}(\mu) - I_{K}^{V^{\pi}}(\pi)\}.
\end{align}

As a consequence, for $\mathbb{P}$-almost all $\omega$,
\begin{enumerate}
    \item $\mu_{n}^{\mathrm{Q}}$ converges in law to $\pi$,
    \item For all $\epsilon > 0$, for any $r > 0$ and $n$ large enough,
    \begin{align*}
     \mathbb{P}_{n, \beta_n}^{\mathrm{Q}}\left( \mathsf{E}_{\mathcal{H}_K}\left(\frac{1}{n}\sum_{i = 1}^{n}\delta_{y_i}\right)> r; \nu_n(\omega)\right)  & = \mathbb{P}_{n, \beta_n}^{\mathrm{Q}}\left( \sqrt{I_{K}(\mu_{n}^{\mathrm{Q}} - \pi)} > r; \nu_n(\omega)\right) \\
     & \leq \exp\left(-\left(\frac{1}{2}-\epsilon\right) \beta_n  r^2\right),
      \end{align*}
\end{enumerate}
where $I_K$ and $\mathsf{E}_{\mathcal{H}_K}$ are respectively defined in \eqref{d:energies} and \eqref{e:worst_case_error}.
\end{theorem}

The proof of the LDP is in \cref{s:proof_ldp}.
The first consequence is immediate, and the second one follows from \cref{e:duality_formula}, along with the fact that $I_{K}(\mu_{n}^{\mathrm{Q}} - \pi) \leq 2(I_{K}^{V_{\pi}}(\mu_{n}^{\mathrm{Q}}) -  I_{K}^{V^{\pi}}(\pi))$; see \citep{rouault_monte_2024}.
Taking a step back, \cref{t:ldp_quenched_bounded} yields tighter simultaneous confidence intervals for Monte Carlo integration of functions in the unit ball of $\mathcal{H}_{K}$ than with points drawn i.i.d from $\pi$ or using MCMC, yet only asymptotically.
We stress the fact that we put no condition on the background $\nu_n$, other than the almost sure convergence to $\pi$. 
Our LDP is thus valid even in the case where $\nu_n$ is the empirical measure of the history of a suitable MCMC chain up to time $n$.
This justifies a common practical heuristics of replacing the intractable potential by a Monte Carlo estimate obtained using a long MCMC run.

From \cref{t:ldp_quenched_bounded}, one might consider taking an inverse temperature $\beta_n$ as large as possible in order to maximize the gain in the tail bound.
Yet, one still has to produce samples from the quenched Gibbs measure with inverse temperature $\beta_n$ in order to provide Monte Carlo estimators.
Heuristically, this becomes more and more difficult as $\beta_n$ becomes large as particles drawn from the Gibbs measure then try to get close to the deterministic minimizers of the energy $H_{n}^{\mathrm{Q}}$ in \cref{d:gibbs_quenched}.
We thus expect a tradeoff between sampling accuracy or complexity and tail bound guarantees governed by $\beta_n$.
However, there is to our knowledge no algorithm to sample approximately from \cref{e:gibbs_measure} that comes with theoretical considerations on this tradeoff, so we give only practical considerations in \cref{s:sampling_and_mala}.

Motivated by simultaneous confidence intervals for the wider class of bounded Lipschitz functions, we now turn to our second main result, covering the case when  $K=g$ is the Coulomb kernel and $d\geq 3$.
We shall need another set of assumptions, always in conjunction with \cref{assum:pi} that guarantees that $\pi$ is compactly supported.



\begin{assumption}
\label{assum:density_equilibrium}
    $\mu_V$ has a continuous and lower bounded density $\mu_V' \geq c > 0$ on the support $S_{\pi}$ of $\pi$.
\end{assumption}

\begin{assumption}
\label{assum:derivatives}
$V$ is $\mathcal{C}^2$ with uniformly bounded second-order derivatives and $\mu_V$ has finite classical entropy.
\end{assumption}

\begin{assumption}
\label{assum:vanishing_density}
The density $\pi'$ of $\pi$ vanishes on the boundary $\partial S_{\pi}$.
\end{assumption}

When $V$ is the quadratic potential $V(z) = \frac{d-2}{2R^d}\lvert z\rvert^2$, for instance, \cref{assum:density_equilibrium} and \cref{assum:derivatives} are satisfied as soon as $S_{\pi} \subset B(0, R)$ since $\mu_V$ is the uniform measure on $B(0, R)$.
We shall also need weights in our Monte Carlo estimator of the potential in the case $K =g$.

\begin{definition}
\label{def:weights}
    Under \cref{assum:pi} and \cref{assum:density_equilibrium}, we define
    \begin{equation}
        \label{e:unnormalized_weights}
        W(x) :=
            \begin{cases}
                \frac{\pi'(x)}{\mu_{V}'(x)} \text{ if $x$ is in the interior of } S_{\pi}\\
                0 \text{ otherwise}
            \end{cases}.
    \end{equation}
    For $x_1, \dots, x_n \in \mathbb{R}^{d},$  we set, for $1\le i \le n,$
    \begin{equation}
        \label{e:weights}
        w(x_i) :=
            \begin{cases}
                    \frac{W(x_i)}{\sum_{i = 1}^{n} W(x_i)} \text{ if there exists } j \text{ such that } \pi'(x_j) >0\\
                    0 \text{ otherwise}
            \end{cases}.
    \end{equation}
\end{definition}
We have the following convergence for the weighted empirical measures.

\begin{proposition}
\label{p:weighted_emp_convergence}
    Under \cref{assum:pi}, \cref{assum:density_equilibrium} and \cref{assum:vanishing_density},
    consider any $x_1, \dots, x_n \in \mathbb{R}^{d}$ such that $\mu_n := \frac{1}{n}\sum_{i = 1}^{n} \delta_{x_i} \rightarrow \mu_V$.
    Then, for $n$ large enough, there exists $j$ s.t. $\pi'(x_j) > 0$, and
    \begin{align}
        \nu_n := \sum_{i = 1}^{n} w(x_i) \delta_{x_i}
    \label{e:weighted_empirical_measure}
    \end{align} 
    is a well defined probability measure for $n$ large enough and $\nu_n {\rightarrow} \pi$ as $n\rightarrow\infty$.
\end{proposition}

\begin{proof}
$W$ is clearly continuous and compactly supported under those assumptions.
$\mu_n = \frac{1}{n}\sum_{i = 1}^{n}\delta_{x_i}$ converges in law to $\mu_V$, so that
\[\frac{1}{n}\sum_{i = 1}^{n} W(x_i) \underset{n \rightarrow + \infty}{\rightarrow} \int W(x)\,\mathrm{d}\mu_V(x) = 1.\] 
Hence, for $n$ large enough, if $x_1, \dots, x_n \sim \mathbb{P}_{n, \beta_n}^{V}$, there exists $i$ s.t $W(x_i) > 0$ (or equivalently $\pi'(x_i) > 0$) and $\nu_n$ is well defined.

For any bounded and continuous test function $f$, we have 
\[\int f\,\mathrm{d}\nu_n = \frac{\frac{1}{n}\sum_{i = 1}^{n} f(x_i) W(x_i)}{\frac{1}{n}\sum_{i = 1}^{n} W(x_i)}.\]
Since $f W$ is bounded and continuous, Slutsky's lemma then implies that $\nu_n$ converges in law to $\pi$.

\end{proof}

Note that since computing $w$ is invariant by multiplication of $\pi'$ by a constant, $\nu_n$ does not depend on the normalization constant of the target $\pi$, and so neither does any linear statistic w.r.t $\nu_n$.
In the following theorem, we chose $x_1, \dots, x_n$ as a realization of a Coulomb gas with external potential $V$, and we show that for this choice of $\nu_n$ in \cref{e:weighted_empirical_measure}, we have the desired LDP for the quenched Gibbs measure.

\begin{theorem}
\label{t:ldp_quenched_coulomb}
Let $d\geq 3$ and $g$ be the Coulomb kernel as in \cref{d:coulomb_kernel}.
Under \cref{assum:pi}, \cref{assum:confinement}, \cref{assum:density_equilibrium}, \cref{assum:derivatives} and \cref{assum:vanishing_density}, consider $\beta_n \geq u n^{1+\delta}$ for $\delta > 0$ and $u > 0$ (in particular $\beta_n \gg n$).
For any $n \geq 2$, let $x_1, \dots, x_n \sim \mathbb{P}_{n, \beta_n}^{V}$ and let $\nu_n$ be as in \cref{e:weighted_empirical_measure}.

Consider the quenched empirical measure $\mu_{n}^{\mathrm{Q}} = \frac{1}{n}\sum_{i = 1}^{n}\delta_{y_i}$, where $y_1, \dots, y_n \sim \mathrm{d}\mathbb{P}_{n, \beta_n}^{\mathrm{Q}, \pi}(.\, ; \nu_n(\omega))$ in \cref{d:gibbs_quenched} with \[0 < \zeta < \frac{\delta \wedge \frac{2}{d}}{2d}.\]
Then for $\mathbb{P}$-almost all $\omega$, $\mu_{n}^{\mathrm{Q}}$ satisfies a large deviations principle as $n\rightarrow + \infty$, at speed $\beta_n$ and with good rate function $\mu \mapsto I_{g}^{V^{\pi}}(\mu) - I_{g}^{V^{\pi}}(\pi)$, whose minimizer is $\pi$.
More precisely, for $\mathbb{P}$-almost all $\omega$,

\begin{align}
   - \underset{\mu \in \overset{\circ}{A}}{\inf\,} \{I_{g}^{V^{\pi}}(\mu) - I_{g}^{V^{\pi}}(\pi)\} &\leq \underset{n \rightarrow +\infty}{\underline{\lim}}\,\frac{1}{\beta_n}\log\mathbb{P}_{n, \beta_n}^{\mathrm{Q}}\left(\mu_{n}^{\mathrm{Q}} \in A; \nu_n(\omega)\right)  \nonumber\\
   &\revv{\leq}\underset{n \rightarrow +\infty}{\overline{\lim}}\,\frac{1}{\beta_n}\log\mathbb{P}_{n, \beta_n}^{\mathrm{Q}}\left(\mu_{n}^{\mathrm{Q}} \in A; \nu_n(\omega)\right) \leq - \underset{\mu \in \overline{A}}{\inf\,} \{I_{g}^{V^{\pi}}(\mu) - I_{g}^{V^{\pi}}(\pi)\}.
\end{align}

As a consequence, for $\mathbb{P}$-almost all $\omega$:
\begin{enumerate}
    \item $\mu_{n}^{\mathrm{Q}}$ converges in law to $\pi$,
    \item There exists $ c > 0 $ such that for any $r > 0$ and $n$ large enough, \[\mathbb{P}_{n, \beta_n}^{\mathrm{Q}}\left( \mathsf{E}_{\mathcal{BL}}\left(\frac{1}{n}\sum_{i = 1}^{n}\delta_{y_i}\right) > r;\, \nu_n(\omega)\right) \leq  \exp(-c \beta_n r^2).\]
\end{enumerate}
\end{theorem}

The proof is again in \cref{s:proof_ldp}.
The first consequence is immediate, and the second one follows from the Coulomb transport inequality \citep[Theorem 1.2]{chafai_concentration_2018}.
Again \cref{t:ldp_quenched_coulomb} gives tighter asymptotic simultaneous asymptotic confidence intervals for Monte Carlo integration of bounded Lipschitz functions than with points drawn i.i.d from $\pi$ or using MCMC.
In that case however, we need the first approximation $\nu_n$ to $\pi$ to be built as well on a Coulomb gas, i.e. to exhibit fast converging properties, this is due to the fact that the quality of $\nu_n$ interplays with the regularization that we need to make to deal with the singularity of $g$.
Moreover, the fact that the regularized Coulomb interaction \cref{e:coulomb_kernel_reg} satisfies \cref{assum:non_singular} will play an important role in the proof.
In particular, in dimensions $1$ and $2$, the natural counterpart to \cref{e:coulomb_kernel_reg} is given by 
\[K_{\zeta}'(x, y) = -\frac{1}{2}\log \left(\left\lvert x - y \right\rvert^2 + n^{-2\zeta}\right),\]
which fails to satisfy \cref{assum:non_singular} since $I_{K}(\mu) > 0$ only for finite signed Borel measures with total mass $0$ \citep{pronzato_bayesian_2020}.
Still, it would be interesting to try to circumvent this issue in the proofs.

We will use a general result on LDPs for Gibbs measures from \citep{garcia-zelada_large_2019} to get \cref{t:ldp_quenched_bounded} and \cref{t:ldp_quenched_coulomb}.
The case of non singular interactions is relatively easy to handle, whereas more control is needed in the Coulomb case because of the regularization.
In particular, the key control will be given by \cref{p:max_pot_weight}.
In the latter, we prove uniform convergence of the regularized potential generated by the Coulomb gas $x_1, \dots, x_n$.
We first begin with an unweighted version of the desired result, which we will use in the proof of the weighted version and which we believe to be interesting as discussed in \cref{s:introduction}.

\begin{proposition}
\label{p:max_pot}
Let $d\geq 3$ and $g$ be the Coulomb kernel.
Under \cref{assum:confinement} and \cref{assum:derivatives}, let $\beta_n \geq u n^{1+\delta}$ for some $\delta > 0$ and $u > 0$.
Let $x_1, \dots, x_n \sim \mathbb{P}_{n, \beta_n}^{V}$ with $K = g$ and $0 < \zeta <  \frac{\delta \wedge \frac{2}{d+2}}{d-2}$, and let $\mu_n$ denote the associated empirical measure.

Then $\mathbb{P}$-almost surely, 
\[A_n := \underset{z \in \mathbb{R}^{d}}{\sup}\left\lvert U_{K_{\zeta}}^{\mu_n}(z) - U_{g}^{\mu_V}(z)\right\rvert \underset{n \rightarrow +\infty}{\rightarrow} 0.\]
\end{proposition}

We give a sketch of proof, to be fleshed out in \cref{s:proof_max_pot}.
We first note that $K_{\zeta}$ satisfies \cref{assum:non_singular}, so that $K_{\zeta}(z, .)$ is in the RKHS $\mathcal{H}_{K_\zeta}$ for all $z$ and  has squared norm $n^{\zeta (d-2)}$ in $\mathcal{H}_{K_\zeta}$.
In particular, the duality formula \eqref{e:duality_formula} yields
\[
    \underset{z \in \mathbb{R}^{d}}{\sup}\,\left\lvert U_{K_{\zeta}}^{\mu_{n}}(z) - U_{K_{\zeta}}^{\mu_V}(z)\right\rvert \leq \sqrt{n^{\zeta(d-2)} I_{K_{\zeta}}(\mu_{n}-\mu_V)}.
\]
Then, one can show (\cref{l:bound_Ih_Ig}) that $I_{K_{\zeta}}$ is always smaller than $I_g$.
Since the points $x_1, \dots, x_n$ are generated from a Coulomb gas with interaction $g$, one should thus be able to use concentration properties of Gibbs measures (\cref{l:concentration}) to show that $n^{\zeta(d-2)} I_g$ goes to zero.
However, $g$ is not integrable at infinity, so that $I_g(\mu_n - \mu_V)$ is infinite.
We bypass this issue by convolving each Dirac $\delta_{x_i}$ in $\mu_n$ with a uniform distribution on a small ball of radius $n^{-\epsilon}$ to make sense of the energy $I_g$.
Denoting by $\mu_{n}^{(\epsilon)}$ this convolved empirical measure, we thus show that 
\[
    n^{\zeta (d-2)} I_{H_{\zeta}}(\mu_{n}^{(\epsilon)} - \mu_V) \leq n^{\zeta(d-2)}I_g(\mu_{n}^{(\epsilon)} - \mu_V) \rightarrow 0
\] 
for a suitable choice of $\epsilon$ and $\zeta$.
It then remains to show that the error terms vanishes when one replaces $U_{K_{\zeta}}^{\mu_n} - U_{g}^{\mu_V}$ by $U_{K_{\zeta}}^{\mu_{n}^{(\epsilon)}} - U_{K_{\zeta}}^{\mu_V}$, which we do in \cref{l:comparison_g_Hzeta} and \cref{l:taylor_no_weights}.

We now state the same result for the weighted measures, the proof of which relies on \cref{p:max_pot}. In the proof of \cref{t:ldp_quenched_coulomb}, we will use the following:
\begin{proposition}
\label{p:max_pot_weight}

Let $d\geq 3$ and $g$ be the Coulomb kernel.
Under \cref{assum:pi}, \cref{assum:confinement}, \cref{assum:density_equilibrium}, \cref{assum:derivatives} and \cref{assum:vanishing_density}, 
let $\beta_n \geq u n^{1+\delta}$ for some $\delta > 0$ and $u > 0$.
Let $x_1, \dots, x_n \sim \mathbb{P}_{n, \beta_n}^{V}$ with $K = g$ and $0 < \zeta < \frac{\delta \wedge \frac{2}{d}}{2d}$, and let $\nu_n = \sum_{i = 1}^{n} w(x_i) \delta_{x_i}$ be the weighted empirical measure defined in \cref{e:weighted_empirical_measure}.

Then $\mathbb{P}$-almost surely, 
\[B_n := \underset{z \in \mathbb{R}^{d}}{\sup}\left\lvert U_{K_{\zeta}}^{\nu_n}(z) - U_{g}^{\pi}(z)\right\rvert \underset{n \rightarrow +\infty}{\rightarrow} 0.\]
\end{proposition}

The constraint on $\zeta$ is upper bounded by $\mathcal{O}(1/d^2)$ both in \cref{p:max_pot} and \cref{p:max_pot_weight}, but it is not expected to be optimal at all.

We give again a sketch of proof, with details to be found in \cref{p:max_pot_weight}.
The beginning of the proof is in the same vein as \cref{p:max_pot}: we show using \cref{l:taylor_no_weights} and \cref{l:taylor_weights} that it is enough (i.e. the error terms vanish) to study $U_{K_{\zeta}}^{\nu_{n}^{(\epsilon)}} - U_{K_\zeta}^{\pi}$ where $\nu_{n}^{(\epsilon)}$ is the weighted empirical measure with each Dirac convolved by a uniform distribution on $B(0, n^{-\epsilon})$.
For that purpose, we need instead of $\nu_{n}^{(\epsilon)}$ to consider $\nu_n^{\epsilon, \eta}$ in which the weights are smoothed versions of W obtained by convolving with a Gaussian mollifier.
Then, one wishes to decouple the weights from the unweighted measures $\mu_{n}^{(\epsilon)}$ and $\mu_V$ to apply once again the concentration properties for the Gibbs measure from which the points are generated.
To do so, a key idea is to rephrase the energy $I_{K}(\mu - \nu)$ in Fourier space as an $L^2$-norm between the characteristic functions of $\mu$ and $\nu$ with respect to a spectral measure induced by $K$.
For $K = K_{\zeta}$ and $K = g$, we give explicit expressions and useful estimates for this spectral measure in \cref{s:preliminaries}.
We then show in \cref{l:fourier_weights} that we can split this $L^2$-norm in two terms: one vanishing term and the other one controlled by quantities of the form $n^{2\zeta d} I_g(\mu_{n}^{(\epsilon)}-\mu_V)$ which converges to $0$ again for suitable choices of $\epsilon$ and $\zeta$.

\section{Discussion and perspectives}

We have shown that a suitable random approximation of the potential can be integrated into large deviations results that support the use of Gibbs measures for Monte Carlo integration.
Several open questions deserve further investigation.

\subsubsection*{Extensions}
First, for exhaustiveness, generalizing \cref{t:ldp_quenched_coulomb} to dimensions $1$ and $2$ could be of interest. 
The difficulty is that the natural replacement for the regularized Coulomb kernel \eqref{e:coulomb_kernel_reg} in dimensions 1 and 2 is not positive semidefinite.

A second interesting question lies within the choice of the size of the background $\nu_n$.
Our \cref{t:ldp_quenched_bounded} is only asymptotic, and does not bring to light the impact of the quality of $\nu_n$; yet the latter matters in practice, as we experimentally show in \cref{s:experiments}.
Quantifying this impact at fixed $n$ would require a concentration inequality in the spirit of \cref{t:ldp_quenched_bounded}.
For the record, one can easily show using the computations of \cite{rouault_monte_2024} that in the setup of \cref{t:ldp_quenched_bounded}, there exists a constant $c > 0$ such that for $\mathbb{P}$-almost all $\omega$ and for any $n \geq 2$ and $r > n^{-1/2}$,
\begin{equation}
    \label{e:concentration_quenched}
    \mathbb{P}_{n, \beta_n}^{\mathrm{Q}}\left(\mathrm{E}_{\mathcal{H}_K}\left(\frac{1}{n}\sum_{i = 1}^{n}\delta_{y_i}\right) > r\right) \leq \exp\left(-c \beta_n r^2  + 2 \beta_n \,\mathrm{E}_{\mathcal{H}_K}\left(\nu_n(\omega)\right)\right).
\end{equation}
If $\nu_n$ is the empirical measure of the history of a Markov chain of size $M_n$, it is known \citep[Proposition 1]{dwivedi_kernel_2021} that $\mathrm{E}_{\mathcal{H}_K}\left(\nu_n\right) = \mathcal{O}(M_{n}^{-1/2})$ with high probability.
Therefore, there exists constants $u_0, u_1$ such that for any $r > u_0 (n^{-1/2} \vee M_{n}^{-1/4})$, with high probability,
\begin{equation}
    \mathbb{P}_{n, \beta_n}^{\mathrm{Q}}\left(\mathrm{E}_{\mathcal{H}_K}\left(\frac{1}{n}\sum_{i = 1}^{n}\delta_{y_i}\right) > r\right) \leq \exp\left(-u_1 \beta_n r^2 \right).
\end{equation}
When $M_n = n$ for instance, this suggests a downgrade in the constraint on $r$ compared to the concentration result of \cite{rouault_monte_2024} with exact knowledge of the potential $U_{K}^{\pi}$, while there is no downgrade when using a larger initial Markov chain $M_n \geq n^2$.
In line with \citep{dwivedi_kernel_2021, chatalic_efficient_2025}, this seems to point out that the background should be a significantly longer Markov chain than the cardinality of the final quadrature. 
\revv{Alternatively}, one could choose a repulsive background with variance reduction, such as the Coulomb gas we used in \cref{t:ldp_quenched_coulomb}.
Still, these recommendations are to be taken with a pinch of salt as our bound~\eqref{e:concentration_quenched} might be crude.

\subsubsection*{A fast CLT for linear statistics}
A stronger argument in favor of using Gibbs measures for Monte Carlo integration would be a central limit theorem with a fast rate.
While obtaining such a result for the singular Coulomb interaction is notoriously difficult \citep{serfaty_gaussian_2023}, one can reasonably hope that this would be easier for a bounded interaction kernel.

\subsubsection*{Sampling from the Gibbs measure}
Now that we have incorporated the approximation of the potential into a fast Monte Carlo result, the elephant in the room is our need for an MCMC algorithm to approximately sample from our Gibbs measures in reasonable time.
The experiments given in \cref{s:experiments} show that this can be an issue if one wants to see smaller integration errors or variances than MCMC in practice.
In the absence of an exact sampler, understanding the sampling complexity of an MCMC algorithm like MALA for this specific target in terms of the number of iterations $T$, the step size, the dimension $d$, the number of particles $n$, and the inverse temperature $\beta_n$ would be important, and it would be interesting to compare the different methods we considered in \cref{s:experiments} with fixed budget.
    Additionally, the recent work of \cite{chen_stationary_2025} suggests that stationary points of the MMD are both computationally tractable and lead to efficient numerical integration of the functions in a large subset of the underlying RKHS.
    For integrating these particular functions, deterministic optimization of the MMD might thus be preferable.
    \revv{Alternatively}, when MALA fails to reach the targeted Gibbs distribution because the gradient of the log target is too small, we may thus be in the vicinity of a stationary point of the MMD, which in turn would not be bad news in terms of numerical integration.
Finally, sampling algorithms other than MALA, tailored to our Gibbs measures, could also achieve better performances.
One could as well think of drawing points from a discrete Gibbs measure supported on a pre-drawn background $\nu_n$, following the two-step rationale of coresets or kernel thinning \citep{dwivedi_kernel_2021}.

\section{Proofs: quenched large deviations}
\label{s:proof_ldp}

In \citep{garcia-zelada_large_2019}, the author established a LDP for Gibbs measures in a quite general framework. In \cref{s:GZ}, we recall
his results in the case when $\frac{\beta_n}{n}$ converges to infinity (that we denote by $\beta_n \gg n$) and detail in particular the conditions that we need to check to apply the results in our context. In \cref{s:quenched_ldp:bounded}, we check these conditions in the case of non-singular kernels and conclude the proof of  Theorem \ref{t:ldp_quenched_bounded}. 
To further prove \cref{t:ldp_quenched_coulomb}, where the kernel is the Coulomb kernel $g$,
we will need \cref{p:max_pot_weight}, whose proof is the main technical input of the present paper and is postponed to  \cref{s:proof_max_pot}.


\subsection{Reminder on LDPs for Gibbs measures, after \citep{garcia-zelada_large_2019}}
\label{s:GZ}

Equip the set $\mathcal P (\mathbb{R}^d)$ of probability measures on $\mathbb{R}^{d}$ with the topology of weak convergence.  
The following theorem is a general large deviations result for Gibbs measures, due to \citet{garcia-zelada_large_2019}.

\begin{theorem}[\citet{garcia-zelada_large_2019}, Corollary 1.3]
    \label{t:general_ldp_gibbs}
    For a reference probability measure $\nu \in \mathcal{P}\left(\mathbb{R}^d\right)$, we consider the unnormalized Gibbs measure
    \[\mathrm{d}\gamma_n(y_1, \dots, y_n) = \exp^{-\beta_n W_n(y_1, \dots, y_n)}\,\mathrm{d}\nu^{\otimes n}(y_1, \dots, y_n),\]
    with $W_n : \left(\mathbb{R}^{d}\right)^n \rightarrow (-\infty, +\infty]$  symmetric and measurable.    
    We extend $W_n$ to all probability measures by defining
    \begin{equation}
         \tilde{W}_n(\mu) = 
            \begin{cases}
            W_n(y_1, \dots, y_n)  \text{ if $\mu$ is atomic with } \mu = \frac{1}{n}\sum_{i = 1}^{n}\delta_{y_i}\\
            +\infty  \text{ otherwise}
            \end{cases}.
    \end{equation}
Consider the following conditions.
%
%

\begin{enumerate}
    \item \label{assum:ldp:stable} \emph{(Stable sequence)} There exists a constant $C \in \mathbb{R}$ such that for all $n$, $W_n \geq C$ uniformly over $\left(\mathbb{R}^{d}\right)^{n}$.
    \item \label{assum:ldp:confining} \emph{(Confining sequence)}
    For any increasing sequence of natural numbers $(n_j)$ and any sequence of probability measures $(\mu_j)$ on $\mathbb{R}^{d}$,
    if there exists a constant $A$ such that $\tilde{W}_{n_j}(\mu_j) \leq A$ for every $j$ then $(\mu_j)$ is relatively compact.
    \item \label{assum:ldp:lower_limit} \emph{(Lower limit)}
    For any sequence of probability measures $(\mu_n)$ on $\mathbb{R}^{d}$ converging to $\mu$, we have $W(\mu) \leq \underline{\lim}\, \tilde{W}_n(\mu_n)$.
    \item \label{assum:ldp:regularity} \emph{(Regularity)}
    For any probability measure $\mu$ on $\mathbb{R}^{d}$ such that $W(\mu) < +\infty$, we can find a sequence of probability measures $(\mu_n)$ in $\mathcal{N}$ such that $\mu_n {\rightarrow} \mu$ and $\overline{\lim}\, W(\mu_n) \leq W(\mu)$, where
    \[\mathcal{N} = \{\xi \in \mathcal{P}(\mathbb{R}^{d}) \, : \, \mathrm{KL}(\xi, \nu) < +\infty,\, \overline{\lim}\, \mathbb{E}_{\xi^{\otimes n}}\left[W_n\right] \leq W(\xi)\},\]
    and $\mathrm{KL}(\xi, \nu)$ is the relative entropy.
\end{enumerate}

Under Conditions \ref{assum:ldp:stable}, \ref{assum:ldp:confining}, \ref{assum:ldp:lower_limit}, and \ref{assum:ldp:regularity}, further assume that $Z_n = \gamma_n(\left(\mathbb{R}^{d}\right)^{n})$ is positive for $n$ large enough.
Define the normalized Gibbs measure $\mathrm{d}\mathbb{P}_{n} = \frac{1}{Z_n}\gamma_n$, as well as $\mu_n = \frac{1}{n}\sum_{i = 1}^{n}\delta_{y_i}$ the associated empirical measure with $y_1, \dots, y_n \sim \mathbb{P}_n$.
Then $\mu_n$ satisfies a LDP as $n \rightarrow +\infty$ with speed $\beta_n$ and good rate function $\mu \mapsto W(\mu) - \underset{\eta \in \mathcal{P}(\mathbb{R}^{d})}{\inf} W(\eta)$.
\end{theorem}


When Conditions~\ref{assum:ldp:stable} to \ref{assum:ldp:regularity} are met, it is said that $W$ is the macroscopic zero-temperature limit of $W_n$.
Our objective is to apply Theorem~\ref{t:general_ldp_gibbs} to the discrete energy $H_{n}^{\mathrm{Q}}$, where we recall from \cref{d:gibbs_quenched} that
\begin{equation}
 \label{e:hnq}
H_{n}^{\mathrm{Q}}(y_1, \dots, y_n) = \frac{1}{2n^2}\sum_{i \neq j} K(y_i, y_j) + \frac{1}{n}\sum_{i = 1}^{n} v(y_i) - \frac{1}{n}\sum_{i = 1}^{n} U_{K}^{\nu_n}(y_i),
\end{equation}
with $v = 0$ inside $B(0, R)$ and $v = \Phi$ outside and where $\nu_n$ is a random measure.
The role of the macroscopic zero-temperature limit $W$ in Theorem~\ref{t:general_ldp_gibbs} would naturally be played by
the energy functional in \cref{d:energies}, that is 
\[I_{K}^{V^{\pi}}(\mu) = \frac{1}{2}\iint K(x, y)\,\mathrm{d}\mu(x)\,\mathrm{d}\mu(y) + \int v\,\mathrm{d}\mu - \int U_{K}^{\pi}\,\mathrm{d}\mu = I_{K}^{v}(\mu) - \int U_{K}^{\pi}\,\mathrm{d}\mu\]
when $I_{K}^{V^{\pi}}(\mu) < \infty$.

The main difficulty in applying Theorem~\ref{t:general_ldp_gibbs} is the dependence of $H_n^{\mathrm{Q}}$ to the random measure $\nu_n$.
\cite{garcia-zelada_large_2019} does use \cref{t:general_ldp_gibbs} to prove large deviations for conditioned Gibbs measures (their Theorem 4.1) and particles in varying environment (their Theorem 4.2), but their framework cannot be directly applied to our case since they either restrict to particles supported on a compact subset, or use a different sign convention for 
the background distributions.
Neither of these settings covers our case, since we wish for a Gibbs measure supported on $\left(\mathbb{R}^{d}\right)^{n}$, and because the conditioned Gibbs measure considered in
\citep{garcia-zelada_large_2019} does not converge to an easily predefinable target due to their sign convention. 
We will however rely on some standard arguments from \citet{garcia-zelada_large_2019} in the upcoming proofs.
Relating their proof to our framework turned out to be relatively easy for non-singular interactions, but tricky in the Coulomb case, where our \cref{p:max_pot_weight} will give the missing key result.

We start with a result that only considers the deterministic part of $H_n^{\mathrm{Q}}$, and can be considered a corollary of \cref{t:general_ldp_gibbs}.

\begin{proposition}
\label{p:ldp_usual_gibbs}
Define
\[W_n^{0}(y_1, \dots, y_n) = \frac{1}{2 n^2}\sum_{i \neq j} K(x_i, x_j) + \frac{1}{n}\sum_{i = 1}^{n} v(y_i)\]
and $I_{K}^{v}(\mu) = \frac{1}{2}\iint K(x, y)\,\mathrm{d}\mu(x)\,\mathrm{d}\mu(y) + \int v\,\mathrm{d}\mu$,
where $K$ is either the Coulomb kernel $g$ or a non-singular interaction kernel that satisfies \cref{assum:non_singular},
and $v$ satisfies $\cref{assum:confinement}$ with constant $c' > 0$. 

Then one can build a reference probability measure $\nu$ and a functional $W_n$ such that 
\begin{itemize}
\item $\gamma_n =  \exp\left(-\beta_n W_n\right)\,\mathrm{d}\nu^{\otimes n}$ is exactly (up to normalization) the Gibbs measure associated to $H_n = W_{n}^{0}$ and confinement $v$, as introduced in \cref{d:gibbs};
\item $(W_{n}, I_{K}^{v})$ satisfies Conditions \ref{assum:ldp:stable}, \ref{assum:ldp:confining}, \ref{assum:ldp:lower_limit} and \ref{assum:ldp:regularity} of \cref{t:general_ldp_gibbs}.
In particular, the empirical measure associated to $\mathbb{P}_{n, \beta_n}^{v}$ with $\beta_n \gg n$ satisfies a LDP with speed $\beta_n$ and rate function $I_{K}^{v} - I_{K}^{v}(\mu_v)$.
\end{itemize}
\end{proposition}

\begin{proof}
We follow the proof of \citet[Theorem 4.11]{garcia-zelada_large_2019}.
Consider the reference probability measure
\[
    \nu = \frac{\exp\left(-c' v(z)\right)}{\int \exp\left(-c' v(z)\right)\,\mathrm{d}z}\,\mathrm{d}z
\]
and set, for any fixed $\epsilon \in (0, 1)$, 
\[G_1(x, y) = K(x, y) + \epsilon v(x)  + \epsilon v(y)\, , \, G_2(x, y) = (1-\epsilon)v(x) + (1-\epsilon)v(y),\]
\[W_{n}^{1}(y_1, \dots, y_n) = \frac{1}{2 n^2} \sum_{i \neq j}G_1(y_i, y_j) \, , \, W_{n}^{2}(y_1, \dots, y_n) = \frac{1}{2 n^2}\sum_{i \neq j} G_2(y_i, y_j),\]
\[a_n  = \frac{1}{1-\epsilon}\left(\frac{1}{n-1}\left(n - \frac{n^2}{\beta_n}c'\right) - \epsilon\right),\]
and $W_n = W_{n}^{1} + a_n W_{n}^{2}$.

The assumptions to apply \cref{t:general_ldp_gibbs} are straightforward given our restrictions on 
$g$ and $v$, using \cite[Proposition 2.8]{chafai_first-order_2014} for the regularity assumption.
The first point of the proposition is just a rewriting of the different terms in the Gibbs measure.
\end{proof}

For the specific choice of $v=0$ in $B(0,R)$ and $v= \Phi$ outside, \cref{assum:confinement} is clearly satisfied, so that we can already apply \cref{p:ldp_usual_gibbs} with $K$ being either Coulomb or satisfying \cref{assum:non_singular}.
Defining $W_n$ and $\nu$ as in \cref{p:ldp_usual_gibbs}, $I_{K}^{v}$ is the zero-temperature macroscopic limit of $W_n$, and we can rewrite the unnormalized quenched Gibbs measure we are interested in as
\begin{multline*}
\exp\left(-\beta_n H_n^{\mathrm{Q}}(y_1, \dots, y_n)\right) \,\mathrm{d}y_1\,\dots\,\mathrm{d}y_n = \\
\exp\left(- \beta_n\left[ W_n(y_1, \dots, y_n) - \frac{1}{n}\sum_{i = 1}^{n} U_{K}^{\nu_n}(y_i)\right]\right)\,\mathrm{d}\nu(y_1)\,\dots\,\mathrm{d}\nu(y_n).
\end{multline*}

If we denote $W_n^{\mathrm{Q}}(y_1, \dots, y_n) = W_n(y_1, \dots, y_n) - \frac{1}{n}\sum_{i = 1}^{n} U_{K}^{\nu_n}(y_i)$, which is symmetric and measurable, we want to show that $I_{K}^{V^{\pi}}$ is the zero-temperature macroscopic limit of $W_n^{\mathrm{Q}}$ $\mathbb{P}$-almost surely.
Indeed, since the partition function of $\mathbb{P}_{n, \beta_n}^{\mathrm{Q}, \pi}$ is well-defined $\mathbb{P}$-almost surely for $n$ large enough (i.e. finite and positive, see \cref{d:gibbs_quenched}), we then will be able to  apply \cref{t:general_ldp_gibbs}, concluding the proof of \cref{t:ldp_quenched_bounded} and \cref{t:ldp_quenched_coulomb} respectively.
It remains to check that $(W_n^{\mathrm{Q}}, I_{K}^{V^{\pi}})$ satisfies Condition~\ref{assum:ldp:stable} to \ref{assum:ldp:regularity} $\mathbb{P}$-almost surely and for $n$ large enough.
This part of the proof will differ in the non-singular case and in the Coulomb case.

\subsection{Proof of \cref{t:ldp_quenched_bounded} (non-singular kernels)}
\label{s:quenched_ldp:bounded}

Let $K$ satisfy \cref{assum:non_singular}. 
We examine the conditions of \cref{t:general_ldp_gibbs} one by one.

First, $0 \leq K \leq C$ by assumption, so that $-U_{K}^{\nu_n} \geq -C$ and $W_n^{\mathrm{Q}}$ satisfies Condition~\ref{assum:ldp:stable} since $W_n$ satisfies it also.

Second, assume that $\tilde W_n^{\mathrm{Q}}(\mu_j) \leq A$ for some constant $A$.
Then the $(\mu_j)_j$ have to be empirical measures $\mu_j = \frac{1}{n_j} \sum_{i = 1}^{n} \delta_{y_{i, j}}$.
Using again $U_{K}^{\nu_j} \leq C$, this yields $W_{n_j}(\mu_j) \leq A + C$. 
Since $W_n$ satisfies Condition~\ref{assum:ldp:confining}, this implies that $(\mu_j)_j$ is relatively compact, and thus $W_n^{\mathrm{Q}}$ satisfies the same condition.

Third, assume that $\mu_n \rightarrow \mu$ with $\mu_n$ empirical measures.
Since $K$ is bounded and continuous and $\nu_n \rightarrow \pi$ $\mathbb{P}$-almost surely, $\iint K\,\mathrm{d}\nu_n\otimes \mu_n \rightarrow \int U_{K}^{\pi}\,\mathrm{d}\mu$ $\mathbb{P}$-almost surely by Fubini.
Since $W_{n}$ satisfies Condition~\ref{assum:ldp:lower_limit}, $W_n^{\mathrm{Q}}$ also satisfies it $\mathbb{P}$-almost surely.

Fourth, let $\mu$ be such that $I_{K}^{V^{\pi}}(\mu) < +\infty$.
Then, since $U_{K}^{\pi}$ is bounded, $I_{K}^{v}(\mu) < +\infty$.
By \cref{p:ldp_usual_gibbs}, $W_n$ satisfies Condition~\ref{assum:ldp:regularity}, so that  there exists a sequence $(\mu_n) \in \mathcal{N}$ such that $\mu_n \rightarrow \mu$ and $\overline{\lim}\, I_{K}^{v}(\mu_n) \leq W(\mu)$
with \[\mathcal{N} = \{ \xi \in \mathcal{P}(\mathbb{R}^{d})\,:\, \mathrm{KL}(\xi, \nu) < +\infty,\, \overline{\lim}\,\mathbb{E}_{\xi^{\otimes n}}\left[W_n\right] \leq I_{K}^{v}(\xi)\}.\]
Since $-U_{K}^{\pi}$ is bounded and continuous by dominated convergence, $\mu \mapsto -\int U_{K}^{\pi}\,\mathrm{d}\mu$ is continuous on $\mathcal{P}(\mathbb{R}^{d})$, so that
$\overline{\lim} I_{K}^{V^{\pi}}(\mu_n) \leq I_{K}^{V^{\pi}}(\mu)$.
Finally, for any $\xi \in \mathcal{P}(\mathbb{R}^{d})$,
\[\mathbb{E}_{\xi^{\otimes n}}\left[-\frac{1}{n}\sum_{i = 1}^{n} U_{K}^{\nu_n}(y_i)\right] = - \int U_{K}^{\nu_n}\,\mathrm{d}\xi = -\int U_{K}^{\xi}\,\mathrm{d}\nu_n\]  
by Fubini. 
Again, this yields $\mathbb{E}_{\xi^{\otimes n}}\left[-\frac{1}{n}\sum_{i = 1}^{n} U_{K}^{\nu_n}(y_i)\right] \rightarrow - \int U_{K}^{\xi}\,\mathrm{d}\pi = - \int U_{K}^{\pi}\,\mathrm{d}\xi$ $\mathbb{P}$-almost surely,
 so that $\overline{\lim}\,\mathbb{E}_{\xi^{\otimes n}}\left[W_n^{\mathrm{Q}}\right] \leq I_{K}^{V^{\pi}}(\xi)$ for any $\xi \in \mathcal{N}$, and in particular for the $(\mu_n)_n$.
 Thus Condition \ref{assum:ldp:regularity} is satisfied for $W_n^{\mathrm{Q}}$, $\mathbb{P}$-almost surely.
This concludes the proof of \cref{t:ldp_quenched_bounded}.

\subsection{Proof of \cref{t:ldp_quenched_coulomb} (Coulomb kernel)}
\label{s:quenched_ldp_coulomb}

Similarly, when $K=g$ is the Coulomb kernel and for our specific choice of $v,$ we check that
all the conditions to apply  \cref{t:general_ldp_gibbs} are satisfied.

Since the interaction is neither bounded nor continuous, we loose the continuity and boundedness properties for the energies compared to \cref{s:quenched_ldp:bounded}, so our approach must differ a bit.
We define 
\[B_n := \underset{z \in \mathbb{R}^{d}}{\sup}\left\lvert U_{K_{\zeta}}^{\nu_n}(z) - U_{g}^{\pi}(z)\right\rvert.\]
We show \cref{t:ldp_quenched_coulomb} given \cref{p:max_pot_weight}, whose proof is the technical crux and is postponed to \cref{s:proof_max_pot}.

We first consider Condition \ref{assum:ldp:lower_limit}.
Assume that the sequence of empirical measures $\mu_n = \frac{1}{n}\sum_{i = 1}^{n} \delta_{z_i}$ converges to $\mu$.
Due to the maximum principle, if $U_{g}^{\pi} \leq M$ on the support of $\pi$ then this holds everywhere \citep[Theorem 1.10, page 71]{landkof_foundations}.
Under the assumption that $\pi$ has compact support and continuous density, $z \mapsto U_{g}^{\pi}(z)$ is continuous \citep[Lemma 4.3]{chafai_first-order_2014}, so that
$U_{g}^{\pi}$ is bounded on the support of $\pi$ and thus everywhere.
Thus, $\mu \longmapsto -\int U_{g}^{\pi}\,\mathrm{d}\mu$ is lower semi-continuous and bounded from below.
In particular we have
\begin{align}
\label{e:uniform_bound_pot_ldp}
C \leq - \int U_{g}^{\pi}\,\mathrm{d}\mu \leq \underline{\lim}\, \left\{-\frac{1}{n} \sum_{i = 1}^{n} U_{g}^{\pi}(z_i)\right\} \leq \underline{\lim}\,\left\{-\frac{1}{n} \sum_{i = 1}^{n} U_{K_{\zeta}}^{\nu_n}(z_i) + B_n\right\}.
\end{align}
Furthermore,
\begin{align}
\label{e:sum_lim_infs}
    \underline{\lim}\, W_{n}(z_1, \dots, z_n) + \underline{\lim}\, \left\{-\frac{1}{n}\sum_{i = 1}^{n} U_{K_{\zeta}}^{\nu_n}(z_i)\right\} \leq \underline{\lim}\, W_{n}^{Q}(z_1, \dots, z_n).
\end{align}
Because $W_n$ satisfies Condition~\ref{assum:ldp:lower_limit} by \cref{p:ldp_usual_gibbs}, and $B_n \rightarrow 0$ $\mathbb{P}$-almost surely by \cref{p:max_pot_weight}, \cref{e:uniform_bound_pot_ldp} and \cref{e:sum_lim_infs} imply that $W_n^{\mathrm{Q}}$ also satisfies it, $\mathbb{P}$-almost surely.

We now consider Condition \ref{assum:ldp:stable}.
Using again \cref{p:max_pot_weight} along with
\begin{align}
    -\frac{1}{n}\sum_{i = 1}^{n} U_{K_{\zeta}}^{\nu_n}(z_i) \geq - \frac{1}{n} \sum_{i = 1}^{n} U_{g}^{\pi}(z_i) - B_n,
\end{align}
and the fact that $U_{g}^{\pi}$ is bounded, $-\frac{1}{n}\sum_{i = 1}^{n} U_{K_{\zeta}}^{\nu_n}(z_i)$ is bounded from below by a constant $C$ uniformly in $n$ for $n$ large enough and $\mathbb{P}$-almost surely.
Since $W_n$ is also uniformly lower bounded, this stays true for $W_n^{\mathrm{Q}}$.

We now consider Condition \ref{assum:ldp:confining}. 
We proceed as for non-singular interactions: if $\tilde{W}_{n_j}^{\text{Q}}(\mu_j) \leq A$,
then $W_{n_j}(\mu_j) \leq A + C$ since we just showed that $-\frac{1}{n}\sum_{i = 1}^{n} U_{K_{\zeta}}^{\nu_n}(z_i)$ is bounded below for $n$ large enough.
So $(\mu_j)$ is relatively compact and $W_n^{\mathrm{Q}}$ satisfies the assumption $\mathbb{P}$-almost surely for $n$ large enough.

We finish by Condition~\ref{assum:ldp:regularity}.
We already know that $U_{g}^{\pi}$ is bounded and continuous so $\overline{\lim}\,I_{g}^{V^{\pi}}(\mu_n) \leq I_{g}^{V^{\pi}}(\mu)$ by the same argument as for non-singular interactions.
Then, for $\xi \in \mathcal{P}(\mathbb{R}^{d})$, 
\begin{align*}
    \mathbb{E}_{\xi^{\otimes n}}\left[-\frac{1}{n}\sum_{i = 1}^{n} U_{K_{\zeta}}^{\nu_n}(y_i)\right] = - \int U_{K_{\zeta}}^{\nu_n}\,\mathrm{d}\xi.
\end{align*}
Using \cref{p:max_pot_weight}, we get 
\[
    \overline{\lim}\,\mathbb{E}_{\xi^{\otimes n}}\left[-\frac{1}{n}\sum_{i = 1}^{n} U_{K_{\zeta}}^{\nu_n}(y_i)\right] \leq - \int U_{g}^{\pi}\,\mathrm{d}\xi,
\]
$\mathbb{P}$-almost surely.
Again, the inequality for the sum of superior limits is in the good direction and we obtain $\overline{\lim}\,\mathbb{E}_{\xi^{\otimes n}}\left[W_n^{\mathrm{Q}}\right] \leq I_{g}^{V^{\pi}}(\xi)$
for any $\xi \in \mathcal{N}$, the set in the Condition~\ref{assum:ldp:regularity} satisfied by $W_n$. 
This concludes the proof of \cref{t:ldp_quenched_coulomb}.

\section{Proofs: uniform convergence for the Coulomb potential field}
\label{s:proof_pot}

We use here the notations of \cref{s:definitions}, in particular the definitions of the Coulomb kernel $g$ from \eqref{e:coulomb_kernel} and its regularized version $K_\zeta$ in \eqref{e:coulomb_kernel_reg}.
We give the proofs of \cref{p:max_pot} and \cref{p:max_pot_weight} along with the statements of the auxiliary lemmas we use.
For conciseness, we delay the proofs of all lemmas to \cref{s:auxiliary_proofs}.

\subsection{Proof of \cref{p:max_pot}}
\label{s:proof_max_pot}

Recall the proof strategy that we sketched after stating \cref{p:max_pot}.
We first bound the differences $U_{g}^{\pi} - U_{K_{\zeta}}^{\pi}$ and $U_{K_{\zeta}}^{\mu_n} - U_{K_{\zeta}}^{\mu_{n}^{(\epsilon)}}$ uniformly in $z$, in \cref{l:comparison_g_Hzeta} and \cref{l:taylor_no_weights}.
Then, we use \cref{l:bound_Ih_Ig} and \cref{l:concentration} to conclude by the concentration properties of the Coulomb gas and with a suitable choice of regularization parameters.
These intermediate results are then assembled into the proof of \cref{p:max_pot}.

\begin{lemma}
\label{l:comparison_g_Hzeta}
Let $\mu$ be a probability measure with compact support and continuous density.
Then for any $\zeta > 0$ and $0 < s < 1$,
\[\underset{z \in \mathbb{R}^{d}}{\sup}\,\left\lvert U_{K_{\zeta}}^{\mu}(z) - U_{g}^{\mu}(z)\right\rvert =\underset{z \in \mathbb{R}^{d}}{\sup}\,\left\lvert \int K_{\zeta}(z, y)\,\mathrm{d}\mu(y) - \int g(z, y)\,\mathrm{d}\mu(y)\right\rvert = \mathcal{O}\left(n^{-4s\zeta/(d+2s)}\right).\]
\end{lemma}

\begin{lemma}
\label{l:taylor_no_weights}
For any $x_1, \dots, x_n \in \mathbb{R}^{d}$, let $\mu_n$ be the empirical measure $\mu_n = \frac{1}{n}\sum_{i = 1}^{n}\delta_{x_i}$.
Let $\epsilon > 0$ and $\mu_{n}^{(\epsilon)} = \mu_n \star \lambda_{B(0, n^{-\epsilon})}$, where $\star$ denotes the convolution, and
$\lambda_{B(0, R)}$ denotes the uniform probability measure on $B(0, R)$.
In other words, 
\[\mu_{n}^{(\epsilon)} = \frac{1}{n}\sum_{i = 1}^{n} \lambda_{B(x_i, n^{-\epsilon})}.\]
Then, $\mathbb{P}$-almost surely,
\[
    \underset{z \in \mathbb{R}^{d}}{\sup}\,\left\lvert U_{K_{\zeta}}^{\mu_n}(z) - U_{K_{\zeta}}^{\mu_{n}^{(\epsilon)}}(z)\right\rvert = \mathcal{O}\left(n^{-2\epsilon + d\zeta}\right).
\]
\end{lemma}

\begin{lemma}
\label{l:bound_Ih_Ig}
For any finite Borel measures $\mu, \nu$ in $\mathbb{R}^{d}$, 
\[I_{K_{\zeta}}(\mu - \nu) \leq I_{g}(\mu - \nu).\]
\end{lemma}

\begin{lemma}[\citealp{chafai_concentration_2018}]
\label{l:concentration}
    Assume that $\beta_n \gg n$, $n\geq 2$, and that \cref{assum:confinement} and \cref{assum:derivatives} hold.
    Then
    \begin{align}
    \label{e:bound_Ig_concentration}
        I_{g}(\mu_{n}^{(\epsilon)} - \mu_V) \leq 2\left(I_{g}^{V}(\mu_{n}^{(\epsilon)}) - I_{g}^{V}(\mu_V)\right).
    \end{align}
    Moreover,
    \begin{align}
    \label{e:bound_Zn_concentration}
        Z_{n, \beta_n}^{V} \geq \exp\left\{- \beta_n I_{g}^{V}(\mu_V) + n\left(\frac{\beta_n}{2n^2} I_{g}(\mu_V) + S(\mu_V)\right)\right\},
    \end{align}
    where $S(\mu_V) = \mathrm{KL}(\mu_V, \lambda)$ is the entropy of $\mu_V$ with respect to the Lebesgue measure.
    Finally, setting $\eta = c' n/\beta_n$ where $c'$ is the constant of \cref{assum:confinement}, and $\epsilon_n = n^{-\epsilon}$ for $\epsilon > 0$, we have
    \begin{align}
    \label{e:bound_Hn_concentration}
        H_n(x_1, \dots, x_n) \geq \eta\frac{1}{n}\sum_{i = 1}^{n} V(x_i) + (1-\eta)\left(I_{g}^{V}(\mu_{n}^{(\epsilon)}) - C\frac{1}{n\epsilon_{n}^{d-2}}-C\epsilon_{n}^{2}\right).
    \end{align} 
    As a consequence, for any Borel set $A \subset \left(\mathbb{R}^{d}\right)^{n}$,
    \begin{align}
        \label{e:bound_concentration_intermediate}
    \log \mathbb{P}_{n, \beta_n}^{V}(A) \leq - (\beta_n - c'n)\,\underset{A}{\inf}\,\left(I_{g}^{V}(\mu_{n}^{(\epsilon)})-I_{g}^{V}(\mu_V)\right) + C n + C\beta_n\left(\frac{1}{n}+\frac{1}{n\epsilon_{n}^{d-2}} + \epsilon_{n}^{2}\right),
    \end{align}
    and, using \cref{e:bound_Ig_concentration}, we also get
    \begin{align}
    \label{e:bound_concentration}
    \log \mathbb{P}_{n, \beta_n}^{V}(A) \leq - \frac{1}{2}(\beta_n - c'n)\,\underset{A}{\inf}\, I_{g}(\mu_{n}^{(\epsilon)}-\mu_V) + C n + C\beta_n\left(\frac{1}{n}+\frac{1}{n\epsilon_{n}^{d-2}} + \epsilon_{n}^{2}\right).
    \end{align}
\end{lemma}

\begin{proof}[Proof of \cref{p:max_pot}]
We can apply \cref{l:comparison_g_Hzeta} with $\mu = \mu_V$ and \cref{l:taylor_no_weights} to get 
\[A_n \leq \underset{z \in \mathbb{R}^{d}}{\sup}\,\left\lvert U_{K_{\zeta}}^{\mu_{n}^{(\epsilon)}}(z)- U_{K_{\zeta}}^{\mu_V}(z)\right\rvert + C n^{-2\epsilon+d\zeta} + C n^{-4s\zeta/(d+2s)},\]
for any $0 < s < 1$ and $\epsilon > 0$.

Then, recall that since $K_{\zeta}$ satisfies \cref{assum:non_singular}, we can build a RKHS with reproducing kernel $K_{\zeta}$.
To wit, the construction starts with 
$$
    \mathcal{H}_0 = \text{span}\left\{ \sum_{j = 1}^{p} a_j K_{\zeta}(z_j, .)\,, a_1, \dots, a_p \in \mathbb{R}\, , z_1, \dots, z_p \in \mathbb{R}^{d}\right\},
$$ 
endowed with the inner product induced by the reproducing property $\left\langle f\,, K_{\zeta}(z, .)\right\rangle = f(z)$ for any $f\in \mathcal{H}_0$ and $z\in \mathbb{R}^{d}$.
Then $\mathcal{H}_{K_{\zeta}}$ is obtained by completion of $\mathcal{H}_{0}$ with respect to the latter inner product, and the reproducing property is true for any $f \in \mathcal{H}_{K_{\zeta}}$.
Clearly, $K_{\zeta}(z, .)$ is an element of $\mathcal{H}_{K_{\zeta}}$ for any fixed $z \in \mathbb{R}^{d}$, and
\[\left\lVert K_{\zeta}(z, .)\right\rVert_{\mathcal{H}_{K_{\zeta}}} = \sqrt{\left\langle K_{\zeta}(z, .)\,,\, K_{\zeta}(z, .)\right\rangle} = \sqrt{K_{\zeta}(z, z)} = \sqrt{n^{\zeta(d-2)}}.\]
Hence, using the duality formula from \cref{e:duality_formula} with $f = K_{\zeta}(z, .)$, we have
\begin{align*}
    \underset{z \in \mathbb{R}^{d}}{\sup}\,\left\lvert \int K_{\zeta}(z, y)\,\mathrm{d}(\mu_{n}^{(\epsilon)}- \mu_V)\right\rvert &\leq \sqrt{n^{\zeta (d-2)}}\underset{\left\lVert f\right\rVert_{\mathcal{H}_{K_{\zeta}}} \leq 1}{\sup}\,\left\lvert \int f\,\mathrm{d}(\mu_{n}^{(\epsilon)}-\mu_V)\right\rvert\\
    &= \sqrt{n^{\zeta(d-2)} I_{K_{\zeta}}(\mu_{n}^{(\epsilon)}-\mu_V)}\\
    &\leq \sqrt{n^{\zeta(d-2)} I_{g}(\mu_{n}^{(\epsilon)}-\mu_V)}.
\end{align*}
where we used \cref{l:bound_Ih_Ig} in the last inequality.
Thus,
\[A_n \leq \sqrt{n^{\zeta(d-2)} I_{g}(\mu_{n}^{(\epsilon)}-\mu_V)} + C n^{-2\epsilon+d\zeta} + C n^{-4s\zeta/(d+2s)}.\]

It remains to make a good choice of $\epsilon > \frac{d \zeta}{2}$ so that all the terms go to zero $\mathbb{P}$-almost surely, using the concentration inequality of \cref{l:concentration}.
In particular, for any fixed $r >0$, we set $A = \left\{(x_1, \dots, x_n)\,:\, n^{\zeta(d-2)} I_{g}\left(\mu_{n}^{(\epsilon)}-\mu_V\right) > r^2 \right\}$ in \cref{l:concentration} to get
\[\log \mathbb{P}_{n, \beta_n}^{V}(A) \leq - c \beta_n n^{-\zeta(d-2)}r^2 + C n + C\beta_n\left(\frac{1}{n} + \frac{1}{n\epsilon_{n}^{d-2}} + \epsilon_{n}^{2}\right),\]
where we used the fact that for any constant $0 < c < 1/2$,  $\beta_n - c'n \geq 2c \beta_n$  for $n$ large enough since $\beta_n \gg n$.
As a consequence, for any fixed $r > 0$ and for $n$ large enough,
\[\log \mathbb{P}_{n, \beta_n}^{V}\left( n^{\zeta(d-2)}I_{g}(\mu_{n}^{(\epsilon)}-\mu_V) > r^2\right) \leq - C \beta_n n^{-\zeta(d-2)}r^2,\]
as soon as \[\zeta(d-2) < \delta \wedge 1 \wedge (1 - \epsilon(d-2)) \wedge 2 \epsilon.\]
Together with $\epsilon > \frac{d \zeta}{2}$, this gives
\begin{align}
    \label{e:constraint_zeta}
    \zeta(d-2) < \delta \wedge 1 \wedge (1 - \epsilon(d-2)) \wedge 2 \frac{d-2}{d} \epsilon.
\end{align}
Choosing $\epsilon = \frac{d}{(d-2)(d+2)}$, the value for which all the constraints on $\epsilon$ match, we see that the constraints \cref{e:constraint_zeta} is always satisfied for $\zeta < \frac{\delta \wedge \frac{2}{d+2}}{d-2}$, which is the assumption we made in \cref{p:max_pot}.
The Borel-Cantelli lemma then yields $n^{\zeta(d-2)}I_{g}(\mu_{n}^{(\epsilon)}-\mu_V) \underset{n \rightarrow +\infty}{\rightarrow} 0$, $\mathbb{P}$-almost surely, which concludes the proof.
\end{proof}

\subsection{Proof of \cref{p:max_pot_weight}}

In this section, we will always work under \cref{assum:pi}, \cref{assum:confinement}, \cref{assum:density_equilibrium}, and \cref{assum:derivatives}, even though most of the following lemmas do not require all of them to be true at once.

We begin in \cref{l:taylor_weights} by a Taylor expansion similar to \cref{l:taylor_no_weights}, but rather dealing with the weights $W$ from \eqref{e:unnormalized_weights}, the motivation of this will become clear afterwards in the proof.
Since $W$ is only assumed to be continuous, we smooth $W$ by convolving it with a Gaussian.

\begin{lemma}
    \label{l:taylor_weights}
    For $\eta > 0$, consider
        \[\rho_{\eta}(x) = \frac{1}{\left(2\pi \eta^2\right)^{d/2}}\exp\left(-\frac{\lvert x\rvert^2}{2\eta^2}\right), \quad x\in \mathbb{R}^{d}.
    \]
    For any $x_1, \dots, x_n \in \mathbb{R}^{d}$ and $\epsilon > 0$, let 
        \[\nu_{n}^{\epsilon, \eta}(A) = \frac{1}{n}\sum_{i = 1}^{n} \int_{A} (W \star \rho_{\eta})(y)\,\mathrm{d}\lambda_{B(x_i, n^{-\epsilon})}(y).\]
    Then, 
    \[
        \underset{z \in \mathbb{R}^{d}}{\sup}\,\left\lvert \frac{1}{n}\sum_{i = 1}^{n}(W \star \rho_{\eta})(x_i)\int K_{\zeta}(z, y)\lambda_{B(x_i, n^{-\epsilon})}(y) - U_{K_{\zeta}}^{\nu_{n}^{\epsilon, \eta}}(z)\right\rvert \leq \eta^{-1} \mathcal{O}\left(n^{-\epsilon + \zeta(d-2)}\right),
    \]
    where the constant in the $\mathcal{O}$ does not depend on $\eta$.
    \end{lemma}

In the following lemma, we bound the difference of potentials for weighted measures.
To do so, we rely on Cauchy--Schwarz to get bounds in term of the $L_2$ norm of the difference of the characteristic functions of the measures.
We then isolate the weights from the $L_2$ norm of the difference of characteristic functions for unweighted measures, which is the interaction energy between the unweighted measures (see \cref{s:auxiliary_proofs}).
Recall that we know how to bound the latter through concentration inequalities (\cref{l:concentration}).

\begin{lemma}
\label{l:fourier_weights}
Recall the definition of $\nu_{n}^{\epsilon, \eta}$ from \cref{l:taylor_weights}, and that $\mu_V'$ is the density of $\mu_V$.
For any $0 < s < 1$ and $1 < \alpha < 1/s$, we have
\begin{align*}
    &\underset{z \in \mathbb{R}^{d}}{\sup}\,\left\lvert U_{K_{\zeta}}^{\nu_{n}^{\epsilon, \eta}}(z) - U_{K_{\zeta}}^{(W \star \rho_{\eta}).\mu_{V}'}(z)\right\rvert \leq \frac{C}{\eta^{d/2}} \lVert W\rVert_{L^{2}(\mathbb{R}^{d})} n^{\zeta(d-2 + 1/s)} d_{\mathrm{BL}}(\mu_{n}^{(\epsilon)}, \mu_V)\\
    & +\frac{C \lVert W\rVert_{L^{1}(\mathbb{R}^{d})}}{\eta^{d}} \left[e^{-\frac{1}{4}n^{\zeta(\alpha -1)}} \left( n^{\zeta(\alpha +1)(d/2 -1)}+n^{\zeta(d-2)}\right) + n^{\zeta(d-2)}\exp\left(-\eta^2 \left(n^{\zeta/s}-n^{\alpha \zeta}\right)^2 /2\right)\right],
\end{align*}
where $d_{\mathrm{BL}}$ is the bounded Lipschitz distance
\[d_{\mathrm{BL}}(\mu, \nu) := \underset{\lVert f\rVert_{\infty} \leq 1,\, \lVert f\rVert_{Lip}\leq 1}{\sup\,}\left\lvert \int f\,\mathrm{d}(\mu - \nu)\right\rvert.\]

In particular, setting for instance $s = 1/2$ and any $1 < \alpha < 2$,
\begin{align}
    \underset{n \rightarrow + \infty}{\lim \sup}\underset{z \in \mathbb{R}^{d}}{\sup}\,\left\lvert U_{K_{\zeta}}^{\nu_{n}^{\epsilon, \eta}}(z) - U_{K_{\zeta}}^{(W \star \rho_{\eta}).\mu_{V}'}(z)\right\rvert &\leq \underset{n \rightarrow + \infty}{\lim \sup\,}\frac{C \lVert W\rVert_{L^{2}(\mathbb{R}^{d})}}{\eta^{d/2}} n^{\zeta d} d_{\text{BL}}(\mu_{n}^{(\epsilon)}, \mu_V)\label{e:fourier_weights}
\end{align}
\end{lemma}

We will now conclude using the concentration as before and trying to optimize the different parameters to make all the error terms go to $0$.

\begin{proof}[Proof of \cref{p:max_pot_weight}]
Recall that we want to show the convergence of 
$$
    B_n := \underset{z \in \mathbb{R}^{d}}{\sup}\left\lvert U_{K_{\zeta}}^{\nu_n}(z) - U_{g}^{\pi}(z)\right\rvert
$$ 
where $\nu_n = \sum_{i = 1}^{n}w(x_i)\delta_{x_i}$.
Since $\sum_{i = 1}^{n} w(x_i) = 1$, the exact same computations as in the proof of \cref{l:taylor_no_weights} yield
\begin{align*}
    \left\lvert U_{K_{\zeta}}^{\nu_n}(z) - \sum_{i = 1}^{n}w(x_i) \int K_{\zeta}(z, y)\,\mathrm{d}\lambda_{B(x_i, n^{-\epsilon})}(y)\right\rvert = \mathcal{O}\left(n^{-2\epsilon + d \zeta}\right),
\end{align*}
uniformly in $z$, for any $\epsilon > 0$.
Together with \cref{l:comparison_g_Hzeta}, this yields, for any $0 < s < 1$,
\begin{align*}
    B_n \leq \underset{z \in \mathbb{R}^{d}}{\sup}\, &\left\vert \sum_{i = 1}^{n}w(x_i) \int K_{\zeta}(z, y)\,\mathrm{d}\lambda_{B(x_i, n^{-\epsilon})}(y) - \int K_{\zeta}(z, y) W(y)\,\mathrm{d}\mu_V(y)\right\rvert\\
    & + C n^{-2\epsilon + d\zeta} + C n^{-4s\zeta/(d+2s)}.
\end{align*}
Moreover, by definition of $w$ and since $0\leq K_{\zeta} \leq g$ and $U_{g}^{\pi}(z) \leq C$ uniformly in $z$ (by the same argument as in \cref{s:quenched_ldp_coulomb}),
\begin{align*}
    \underset{z \in \mathbb{R}^{d}}{\sup}\,& \left\vert \sum_{i = 1}^{n}w(x_i) \int K_{\zeta}(z, y)\,\mathrm{d}\lambda_{B(x_i, n^{-\epsilon})}(y) - \int K_{\zeta}(z, y) W(y)\,\mathrm{d}\mu_V(y)\right\rvert \\
    &\leq \frac{n}{\sum_{i = 1}^{n}W(x_i)} \underset{z \in \mathbb{R}^{d}}{\sup}\,\left\lvert \frac{1}{n}\sum_{i = 1}^{n} W(x_i) \int K_{\zeta}(z, y)\,\mathrm{d}\lambda_{B(x_i, n^{-\epsilon})}(y) - \int K_{\zeta}(z, y)W(y)\,\mathrm{d}\mu_V(y)\right\rvert \\
    & \quad + \left\lvert 1 - \frac{n}{\sum_{i = 1}^{n}W(x_i)}\right\rvert C.
\end{align*}
Recall that $\frac{1}{n}\sum_{i = 1}^{n} W(x_i) \rightarrow 1$, $\mathbb{P}$-a.s.
Moreover, for any $\eta > 0$, 
\begin{align}
    \label{e:first_bound_difference_weighted_pot}
    \underset{z \in \mathbb{R}^{d}}{\sup} \,& \left\lvert \frac{1}{n}\sum_{i = 1}^{n} W(x_i) \int K_{\zeta}(z, y)\,\mathrm{d}\lambda_{B(x_i, n^{-\epsilon})}(y) - \int K_{\zeta}(z, y)W(y)\,\mathrm{d}\mu_V(y)\right\rvert\\\nonumber
    &\leq \underset{z \in \mathbb{R}^{d}}{\sup}\,\left\lvert \frac{1}{n}\sum_{i = 1}^{n} (W\star \rho_{\eta})(x_i) \int K_{\zeta}(z, y)\,\mathrm{d}\lambda_{B(x_i, n^{-\epsilon})}(y) - \int K_{\zeta}(z, y)(W\star \rho_{\eta})(y)\,\mathrm{d}\mu_V(y)\right\rvert\\\nonumber
    & \quad + \left\lVert W - W\star \rho_{\eta}\right\rVert_{\infty} \underset{z \in \mathbb{R}^{d}}{\sup}\, \left(\left\lvert U_{K_{\zeta}}^{\mu_{n}^{(\epsilon)}}(z)\right\rvert + \left\lvert U_{K_{\zeta}}^{\mu_V}(z)\right\rvert\right).
\end{align}
But $U_{K_{\zeta}}^{\mu_V}(z) \leq U_{g}^{\mu_V}(z) \leq C$ for the same reasons as $U_{g}^{\pi} \leq C$, so that using \cref{l:taylor_weights}, the right-hand side of \cref{e:first_bound_difference_weighted_pot} is bounded by
\begin{align*}
    \underset{z \in \mathbb{R}^{d}}{\sup}\, &\left\lvert U_{K_{\zeta}}^{\nu_{n}^{\epsilon, \eta}}(z) - U_{K_{\zeta}}^{(W\star \rho_{\eta}).\mu_{V}'}(z)\right\rvert\\
    &+\eta^{-1} C n^{-\epsilon + \zeta(d-2)} + \left\lVert W - W\star \rho_{\eta}\right\rVert_{\infty}\left(\underset{z \in \mathbb{R}^{d}}{\sup}\,\left\lvert U_{K_{\zeta}}^{\mu_{n}^{(\epsilon)}}(z) - U_{K_{\zeta}}^{\mu_V}(z)\right\rvert + 2C\right).
\end{align*}
Recall now from the proof of \cref{p:max_pot} in \cref{s:proof_max_pot} --especially from \cref{e:duality_formula}-- that 
$$
    \underset{z \in \mathbb{R}^{d}}{\sup}\,\left\lvert U_{K_{\zeta}}^{\mu_{n}^{(\epsilon)}}(z) - U_{K_{\zeta}}^{\mu_V}(z)\right\rvert \leq \sqrt{n^{\zeta (d-2)} I_{K_{\zeta}}(\mu_{n}^{(\epsilon)} - \mu_V)}.
$$
Using \cref{l:fourier_weights} and \cref{l:bound_Ih_Ig}, we thus get that, for any $\epsilon > (\frac{d}{2}\vee (d-2))\zeta$,
\begin{align*}
    \underset{n \rightarrow + \infty}{\lim \sup}\, B_n \leq  \underset{n \rightarrow + \infty}{\lim \sup}\, C_{\eta, W}\left[ n^{\zeta d} d_{\text{BL}}(\mu_{n}^{(\epsilon)}, \mu_V) + \sqrt{n^{\zeta (d-2)} I_{g}(\mu_{n}^{(\epsilon)}-\mu_V)}\right] + 2C \lVert W - W\star \rho_{\eta}\rVert_{\infty}.
\end{align*}
It now remains to use the concentration inequality from \cref{l:concentration} to show that under the right constraint on $\zeta$ and $\epsilon$, the first term goes to $0$ $\mathbb{P}-$a.s.
Using \citep[Theorem 1.2]{chafai_concentration_2018} and \cref{e:bound_Ig_concentration}, it is enough to show that $n^{2 \zeta d}(I_{g}^{V}(\mu_{n}^{(\epsilon)})-I_{g}^{V}(\mu_V)) \rightarrow 0$.

Using \cref{l:concentration}, we know that for any $r >0$, setting 
$$
    A = \{(x_1, \dots, x_n) : n^{2 \zeta d} \left(I_{g}^{V}(\mu_{n}^{(\epsilon)})-I_{g}^{V}(\mu_V)\right) > r^2\},
$$ we have
\begin{align*}
    \log \mathbb{P}_{n, \beta_n}^{V}(A) \leq - c\beta_n n^{-2\zeta d}r^2 + Cn + C\beta_n\left(\frac{1}{n} + \frac{1}{n\epsilon_{n}^{d-2}} + \epsilon_{n}^{2}\right)
\end{align*}
with $\epsilon_n = n^{-\epsilon}$.
Thus, using the Borel--Cantelli lemma, $n^{2 \zeta d} \left(I_{g}^{V}(\mu_{n}^{(\epsilon)})-I_{g}^{V}(\mu_V)\right) \rightarrow 0$ $\mathbb{P}$-a.s. upon choosing $0 < \epsilon < \frac{1}{d-2}$ and
\begin{align*}
    2 \zeta d < \delta \wedge 1 \wedge (1-(d-2)\epsilon)\wedge 2\epsilon.
\end{align*}
Incorporating the constraint $\epsilon > (d/2 \vee (d-2))\zeta$ yields
\begin{align*}
    \zeta <& \frac{\delta \wedge 1}{2 d} \wedge \frac{1}{2d}(1-(d-2)\epsilon) \wedge \left(\frac{1}{d}\wedge \frac{2}{d}\wedge \frac{1}{d-2}\right)\epsilon\\
    &= \frac{\delta }{2d} \wedge \frac{1}{2 d}(1-(d-2)\epsilon) \wedge \frac{1}{d}\epsilon.
\end{align*}
Setting $\epsilon = \frac{1}{d} < \frac{1}{d-2}$, to make the two constrains coincide,
it is thus enough to chose 
$$
    0 < \zeta < \frac{\delta \wedge 2/d}{2d},
$$ 
which is the assumption we made in \cref{p:max_pot_weight}.
With this choice, we thus get
\[\underset{n \rightarrow +\infty}{\lim\sup}\, B_n \leq 2C \lVert W - W\star \rho_{\eta}\rVert_{\infty}.\]
Letting $\eta \rightarrow 0$, we finally get the result, since $W$ is continuous and compactly supported under \cref{assum:vanishing_density}.
\end{proof}


%


\subsection*{Acknowledgements}
We thank Rapha\"el Butez for stimulating discussions around LDPs and Gibbs measures.
We acknowledge support from ERC grant Blackjack ERC-2019-STG-851866, ANR grant Baccarat ANR-20-CHIA-0002 and Labex CEMPI (ANR-11-LABX-0007-01).



\printbibliography

\appendix
\section{Experiments}
\label{s:experiments}

After discussing how we approximately sample from Gibbs measures, we illustrate our main results on toy examples and a Bayesian classification task.

\subsection{Approximately sampling from the Gibbs measure}
\label{s:sampling_and_mala}
There is no known exact sampling algorithm to draw points from Gibbs measures of the form
\begin{align}
\label{e:gibbs_measure}
    \mathrm{d}\mathbb{P}_{n, \beta_n}^{V}(y_1, \dots, y_n) = \frac{1}{Z_{n, \beta_n}^{V}}\exp\left(-\beta_n H_n(y_1, \dots, y_n)\right)\,\mathrm{d}y_1\,\dots\mathrm{d}y_n,
\end{align}
with
\begin{align}
\label{e:H_n}
    H_n(y_1, \dots, y_n) = \frac{1}{2 n^2}\sum_{i \neq j} K(y_i, y_j) + \frac{1}{n}\sum_{i = 1}^{n}V(y_i),
\end{align}
and
\begin{align}
\label{e:jellium}
    V(z) = V^{\pi}(z) :=  - U_{K}^{\pi}(z) + \Phi(z), \quad z\in\mathbb{R}^d.
\end{align}
We thus rely on MCMC to produce points $y_{1}^{T}, \dots, y_{n}^{T}$ that are approximately distributed according to \cref{e:gibbs_measure} after $T$ iterations of an MCMC kernel.
More precisely, we use Metropolis-Adjusted Langevin (MALA; \cite{robert_monte_2004}) updates targeting \cref{e:gibbs_measure}. 
Formally, these are Metropolis--Hastings updates with Gaussian proposal given by
\begin{align}
    \label{e:mala_updates}
    y_{1:n} \,|\, y_{1:n}^{(t)} \sim \mathcal{N}\left(y_{1:n}^{(t)}-\alpha \beta_n \nabla H_n(y_{1:n}^{(t)}),\, 2\alpha I_{dn}\right),
\end{align}
where $y_{1:n}$ is short for $y_1, \dots, y_n$ and $\alpha = \alpha_0 \beta_{n}^{-1}$, with
$\alpha_0$ tuned so as to reach an acceptance proportion close to $50\%$.
This choice comes from the fact that we want to avoid mixing times to depend on $\beta_n$ in warm start cases, following heuristics from \cite{dwivedi_log-concave_2019}; see \citep{rouault_monte_2024} for more details.
Regarding the computational complexity, each of the $T$ MALA iterations requires evaluating $H_n$ and its gradient.
Each of the $T$ iterations thus costs $\mathcal{O}(n^2)$ for the target \eqref{e:gibbs_measure}, and $\mathcal{O}(n^2 + n M_n)$ for the quenched Gibbs measure of Definition 2.5, where $M_n$ is the number of atoms of $\nu_n$ used to approximate $\pi$ in \cref{e:jellium}.

We consider a toy example where $d=3$ and $n=500$. 
The target distribution $\pi$ is a centered Gaussian distribution with covariance $\sigma^2 I_d$, truncated at $\lvert x \rvert > 5 \sigma$, and with $\sigma = 0.5$.
We consider the quenched Gibbs measure with $\beta_n = n^2$, and interaction kernel the regularized Riesz kernel $K_{s, \epsilon}$ given by
\begin{align}
    \label{e:riesz_kernel_reg}
    K_{s, \epsilon}(x, y) = \frac{1}{\left(\lvert x -y \rvert^2 + \epsilon^2\right)^{s/2}},
\end{align}
with $s = d-2$ and $\epsilon = 0.1$.
As for the background, we set $\nu_n = \frac{1}{M_n}\sum_{i = 1}^{M_n}\delta_{x_i}$ where $x_1, \dots, x_{M_n}$ is a realization of the history of a Metropolis--Hastings chain targeting $\pi$, with $M_n = 1\,000$.

In \cref{f:points_gibbs_n2}, we show the first two spatial coordinates of the $50\,000$-th iterate of a MALA chain approximating the Gibbs measure.
For comparison, in \cref{f:points_mcmc}, we show the first two spatial coordinates of an approximate sample $y_1, \dots, y_n$ from $\pi$, obtained as the history $y_1, \dots, y_n$ of a Metropolis--Hastings chain targeting $\pi$ with Gaussian proposal $N(0, \alpha I_d)$, where $\alpha$ is tuned to reach $50\%$ acceptance, and we removed $5\,000$ burn-in iterations.

\begin{figure}
    \centering
    \subfloat[Gibbs, $\beta_n = n^2$ \label{f:points_gibbs_n2}]{
    \includegraphics[width = \threefig]{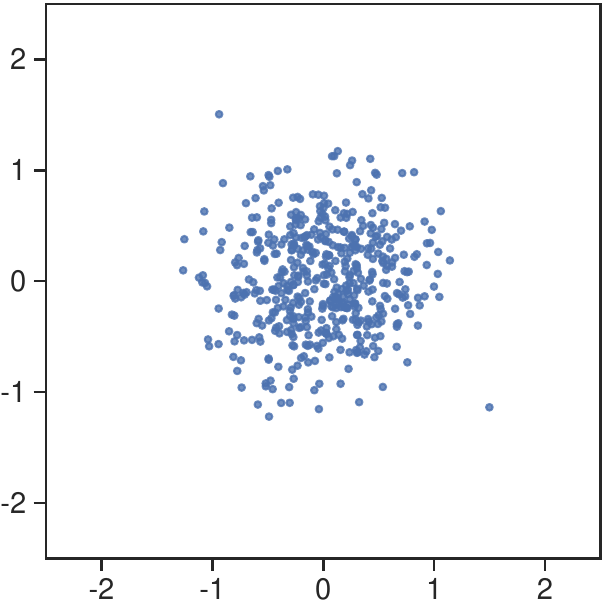}
    }
    \subfloat[MCMC \label{f:points_mcmc}]{
    \includegraphics[width = \threefig]{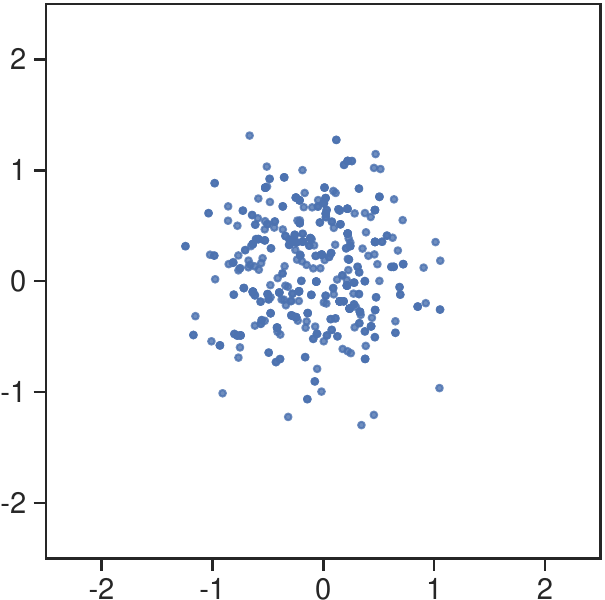}
    }
    \subfloat[Gibbs, $\beta_n = n^2$, 2D uniform target $\pi$ \label{f:points_gibbs_n2_unif}]{
    \includegraphics[width = \threefig]{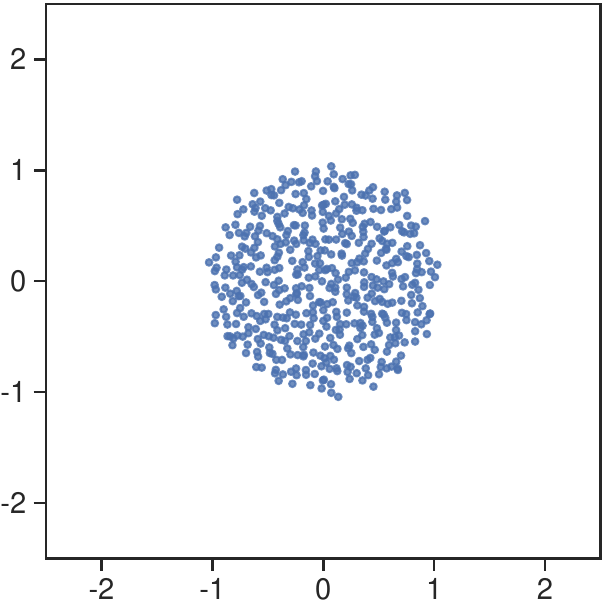}
    }
    \caption{
        First two coordinates of an approximate sample from (\ref{f:points_gibbs_n2}) the quenched Gibbs measure with $\beta_n=n^2$ and 3D truncated Gaussian target $\pi$, (\ref{f:points_mcmc}) MCMC with 3D truncated Gaussian target $\pi$, and (\ref{f:points_gibbs_n2_unif}) the quenched Gibbs measure with $\beta_n = n^2$ and 2D uniform target $\pi$.
    }
    \label{f:plot_points_mcmc_gibbs}
\end{figure}

We see in \cref{f:plot_points_mcmc_gibbs} that both methods stay confined in the regions of large probability under the target $\pi$, as expected.
Moreover, the approximated Gibbs samples in \cref{f:points_gibbs_n2} are more diverse than a Metropolis-Hastings sample, let alone because about half the proposed points in the latter are rejected, explaining the many superimposed disks in \cref{f:points_mcmc}. 
On top of these rejections causing autocorrelation in the MCMC chain, we expect the interaction term in \eqref{e:gibbs_measure} to favor diverse configurations.
However, this promotion of diversity is balanced by 1) the use of an MCMC chain as a confining background, 2) the fact that we truncate the interaction kernel, 3) the constraint that a significant proportion of the points should remain in the bulk of the Gaussian target, as well as 4) the fact that we project a 3-dimensional sample onto a plane.
This explains why repulsiveness is less visually striking in Figure~\ref{f:points_gibbs_n2} than in \cref{f:points_gibbs_n2_unif}, where we plot for reference an approximate Gibbs sample when the target $\pi$ is the 2D uniform measure on $B(0, 1)$, with interaction kernel $-\frac{1}{2}\log(\lvert x - y\rvert^2)$, exact potential $U^{\pi}$ and $T = 5\,000$ MALA iterations.

\begin{figure}
    \centering
    \subfloat[$T = 50\,000$ \label{f:points_gibbs_n3_50k}]{
    \includegraphics[width = \threefig]{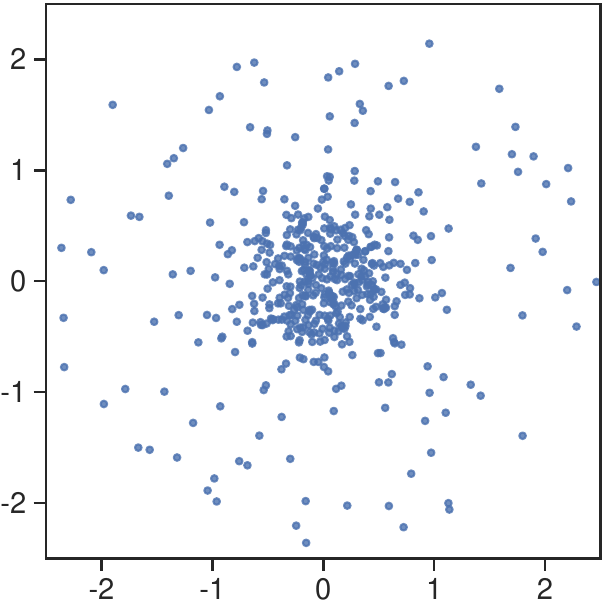}
    }
    \subfloat[$T = 100\,000$ \label{f:points_gibbs_n3_100k}]{
    \includegraphics[width = \threefig]{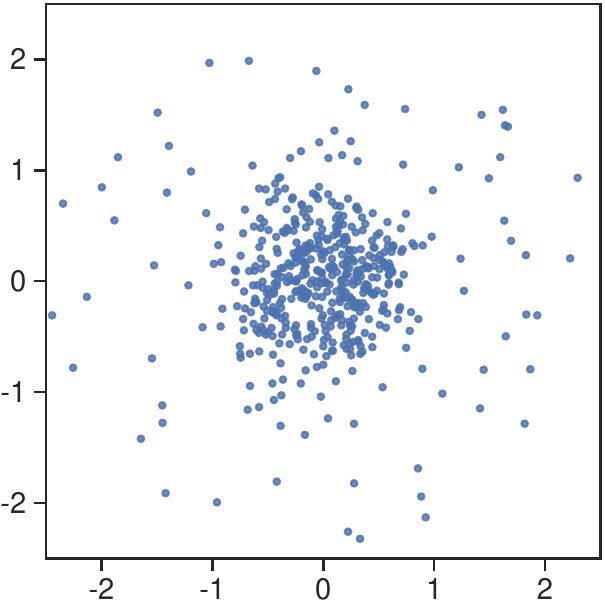}
    }
    \subfloat[$T = 200\,000$ \label{f:points_gibbs_n3_200k}]{
    \includegraphics[width = \threefig]{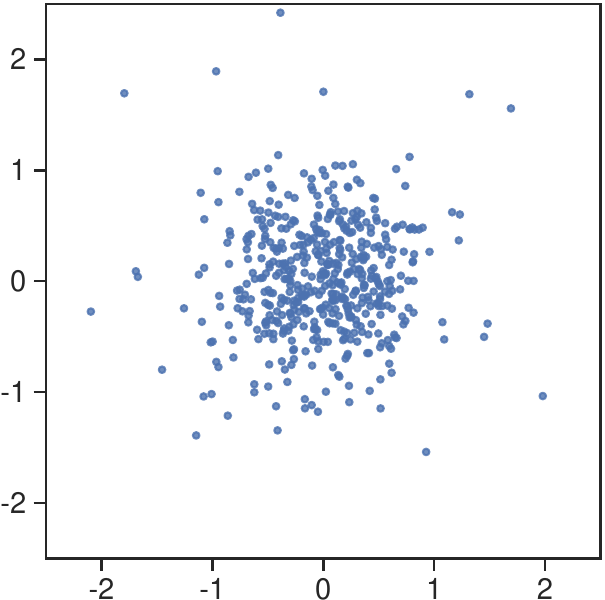}
    }
    \caption{First two coordinates of an approximate sample from the quenched Gibbs measure, with $\beta_n = n^3$ and increasing number $T$ of MALA iterations.}
    \label{f:plot_points_gibbs_n3}
\end{figure}

In \cref{f:plot_points_gibbs_n3}, we plot the first two coordinates in the same setup as \cref{f:points_gibbs_n2}, with the same realization of the background, but this time with an inverse temperature of $\beta_n = n^3$.
In \cref{f:points_gibbs_n3_50k} and \cref{f:points_gibbs_n3_100k}, we see that points spread a lot in the domain of the target, even in regions of low probability under $\pi$. 
In light of \cref{f:points_gibbs_n3_200k}, where outer regions of low probability under $\pi$ now contain a matching fraction of the points, one can understand the overspread in Figures \ref{f:points_gibbs_n3_50k} and \ref{f:points_gibbs_n3_100k} as a consequence of MALA needing more iterations to yield an accurate approximation to its Gibbs target.
This supports the intuitive claim that higher $\beta_n$ comes with a worse mixing time for MALA: the lower the temperature, the more difficult it is to approximately sample a Gibbs measure. 
There thus seems to be a tradeoff in choosing a large $\beta_n$ for better statistical guarantees, and keeping $\beta_n$ low to preserve tractable approximate sampling procedures.
There are to our knowledge no theoretical results that enable to quantify this tradeoff, and one could hope that Markov kernels specifically designed for \cref{e:gibbs_measure} could yield better performance, e.g. smaller confidence intervals for linear statistics.
While those are important questions, we do not address them in this paper.
In practice, we recommend $\beta_n = n^2$, which already yields improved statistical guarantees over independent and MCMC quadratures, and has remained computationally tractable in our experiments.




\subsection{Worst-case integration error and variance}
\label{s:worst-case-variance}

We consider the same setting where $\pi$ is a three-dimensional Gaussian distribution $\mathcal{N}(0, \sigma^2 I_{3})$, truncated at $\lvert x \rvert > 5 \sigma$ with $\sigma = 0.5$. 
We estimate the worst-case integration error of various quadratures $\mu_n := \frac{1}{n}\sum_{i = 1}^{n}\delta_{y_i}$ as $n$ grows.

We first evaluate $\mathrm{E}_{\mathcal{H}_K}(\mu_n) = \sqrt{I_{K}(\mu_n - \pi)}$ when $K = K_{s, \epsilon}$ again with $s = d-2$ and $\epsilon = 0.1$.
In particular, we compute 
\[
    I_{K}(\mu_n - \pi) - I_{K}(\pi) = I_{K}(\mu_n)-2I_{K}(\mu_n, \pi),
\] 
using a long MH chain targeting $\pi$ to approximate the integral with respect to $\pi$, with size $M = 90,000$ and $10,000$ burn-in iterations.
We show in \cref{f:comparison_energies_mcmc} how $\mathrm{E}_{\mathcal{H}_K}(\mu_n)$ decays when $n$ grows for 
i) $\mu_{n}^{\mathrm{MCMC}}$, where the $(y_i)$ is the history of a MH chain targeting $\pi$ after $5000$ burn in iterations,
ii) $\mu_{n}^{\mathrm{Q}}$, where the $(y_i)$ are an approximate sample from the quenched Gibbs measure with $T = 10\,000$ iterations, temperature $\beta_n = n^2$, kernel $K = K_{s, \epsilon}$ and $\nu_n = \frac{1}{M_n}\sum_{i = 1}^{M_n} \delta_{x_i}$ where $M_n = 1,000$ and the $x_i$ are the history of a MH chain targeting $\pi$ again,
iii) $\mu_{n}^{\mathrm{KT}}$ where the $(y_i)$ are a coreset of size $n$ of an $n^2$ MCMC chain $x_1, \dots x_{n^2}$ targeting $\pi$. 
In the latter method, the coreset if obtained via the kernel thinning algorithm of \citet{dwivedi_kernel_2021, dwivedi_generalized_2024}, who proved that the $(y_i)_{i = 1, \dots, n}$ preserve the same sample quality as the original $(x_i)_{i = 1, \dots, n^2}$ with high probability.
We compute the $90\%$-quantile of the worst-case error, over $100$ independent realizations of the quadrature for  each method and each cardinality $n$, with fixed background for ii) and iii).

\begin{figure}
    \centering
        \subfloat[Decay of the energy for MCMC, Gibbs and kernel thinning. \label{f:comparison_energies_mcmc}]{
        \includegraphics[width = \twofig]{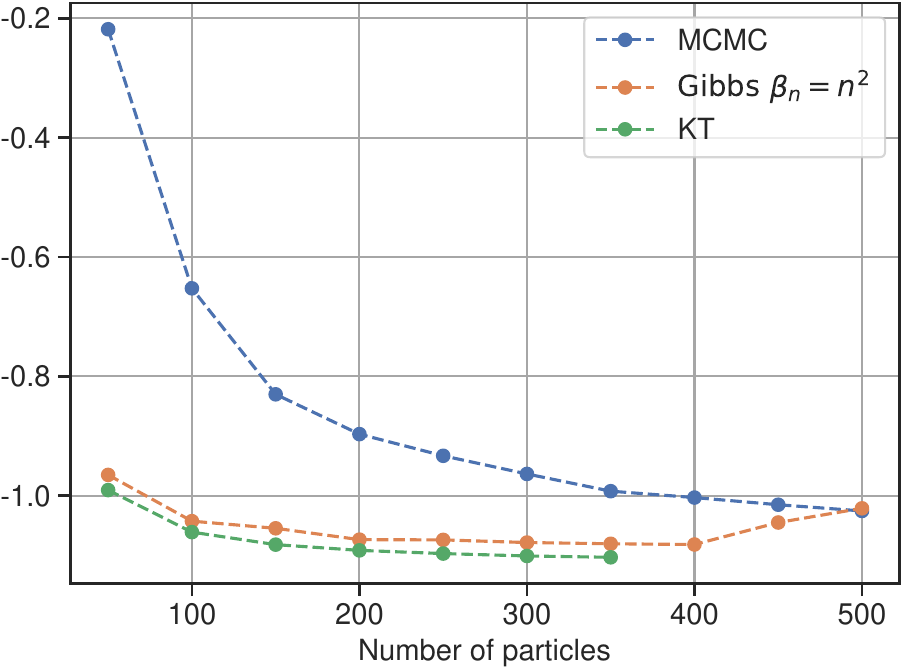}
        }
        \subfloat[Proportions for Bayesian classification. \label{f:proportions_bayesian_classification}]{
    \includegraphics[width = \twofig]{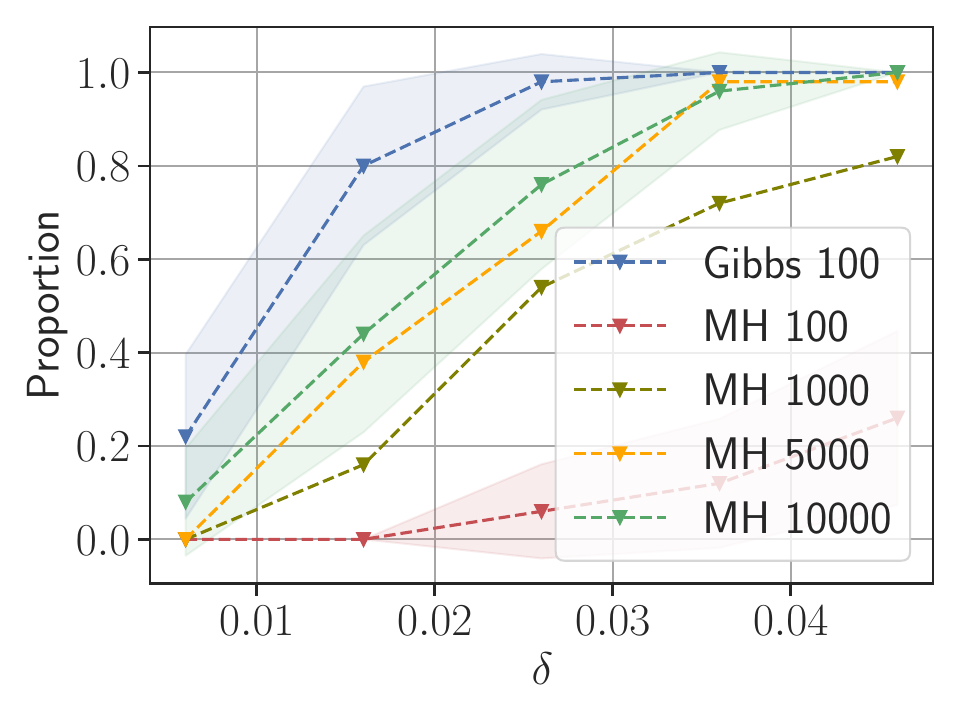}
    }
    \caption{(\ref{f:comparison_energies_mcmc}) Quantile of order $90\%$ of $I_{K}(\mu_n - \pi) - I_{K}(\pi)$, in log scale, for points obtained via MCMC, the quenched Gibbs measure, and kernel thinning.
    (\ref{f:proportions_bayesian_classification}) Proportion of runs where the $M$ estimators fall simultaneously in the confidence regions of size $\delta$.}
    \label{f:energies_confidence}
\end{figure}

As expected, we see that both kernel thinning and our method outperform MCMC in terms of worst-case integration error, at least up to $n = 350$ particles.
For kernel thinning, we used the vanilla version presented in \citep{dwivedi_kernel_2021} which requires to store all the $n^4$ pairwise interactions for the underlying MCMC background of size $n^2$, which is why we could not run the experiment for more than $n = 350$ due to storage issues.
Note however that recent refinements of this algorithm might help circumvent this issue \citep{shetty_distribution_2022}.

Beyond $n = 350$, we observe that the energy increases after for the quenched Gibbs measure, becoming similar to the MCMC error, which we attribute again to the fact that $T = 10\,000$ MALA iterations are not enough to approximate the Gibbs measure for that large a number of particles. 

We now provide further experiments that investigate the impact of the background, as well as a conjectured faster decaying variance for linear statistics.

In \cref{f:comparison_energies_gibbs_background}, we run the same experiment as in \cref{s:worst-case-variance}, where we vary the size $M_n$ of the MCMC background $(x_i)$ to see whether there is an impact on the worst-case integration error.
We also add to the comparison a background $\nu_n$ obtained via reweighting the realization of a Coulomb gas $x_1, \dots, x_{n}$ as in Definition (2.8), with quadratic potential $V$ so that $\mu_V$ is the uniform measure on $B(0, 5\sigma)$.
We use $T = 10\,000$ iterations of MALA to approximately sample that Coulomb background.

\begin{figure}
    \centering
    \includegraphics[width=0.8\textwidth]{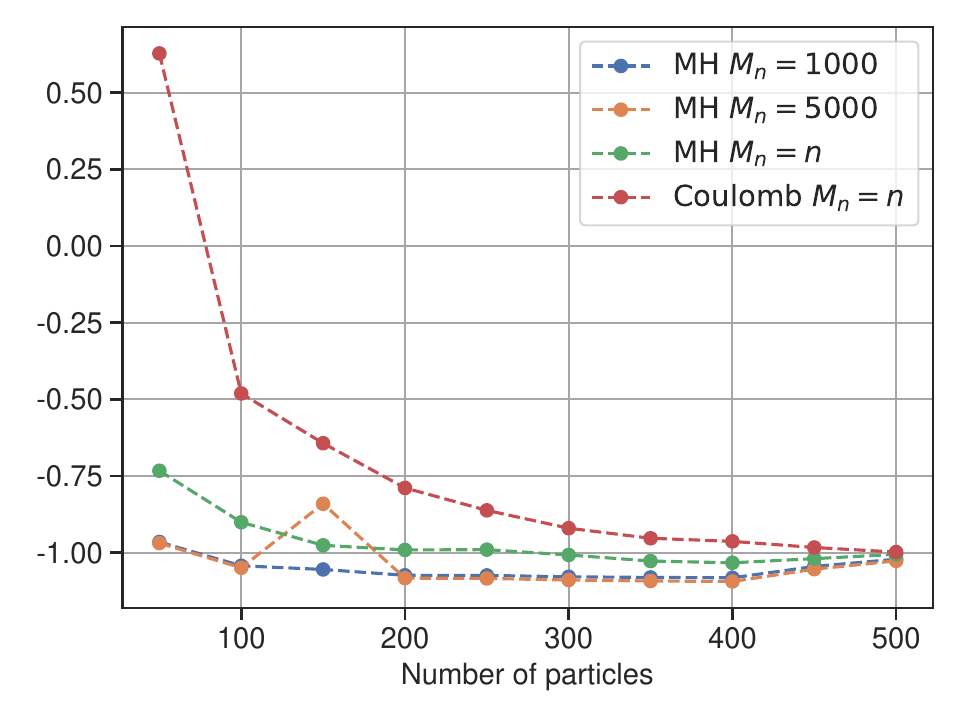}
    \caption{Quantile of order $90\%$ of $I_{K}(\mu_n - \pi) - I_{K}(\pi)$ for different background of the quenched Gibbs measure.}
    \label{f:comparison_energies_gibbs_background}
\end{figure}

When the background is drawn using MCMC, we see a slight improvement between a background of size $n$ and a background of (larger) fixed size $1000$ or $5000$.
When the background is drawn from a first well spread Coulomb gas with reweighting, we observe a small loss in terms of energy, especially for small values of $n$. 
We conjecture that this is due to the quality of the estimation when using weights to approximate the potential with a small number of points $n$.

Then, we look at the effect of increasing the temperature schedule $\beta_n$ in \cref{f:comparison_energies_gibbs_temperature_mcmc} for MCMC background with size $M_n = 1000$, and in \cref{f:comparison_energies_gibbs_temperature_coulomb} for a reweighted Coulomb background $x_1, \dots, x_{n}$ with uniform equilibrium measure.
Again, we clearly see a downgrade in terms of energy for lower temperature schedules which is again due to the fact that $T = 10\,000$ are not enough to sample from the Gibbs measure in those regimes as we saw in \cref{f:plot_points_gibbs_n3}.
Running the same experiments with $T = 200\,000$ MALA iterations for instance is very costly, and one could argue that it is not reasonable for a three-dimensional truncated Gaussian target.

\begin{figure}[!ht]
    \centering
    \subfloat[MCMC background \label{f:comparison_energies_gibbs_temperature_mcmc}]{%
        \includegraphics[width=\twofig]{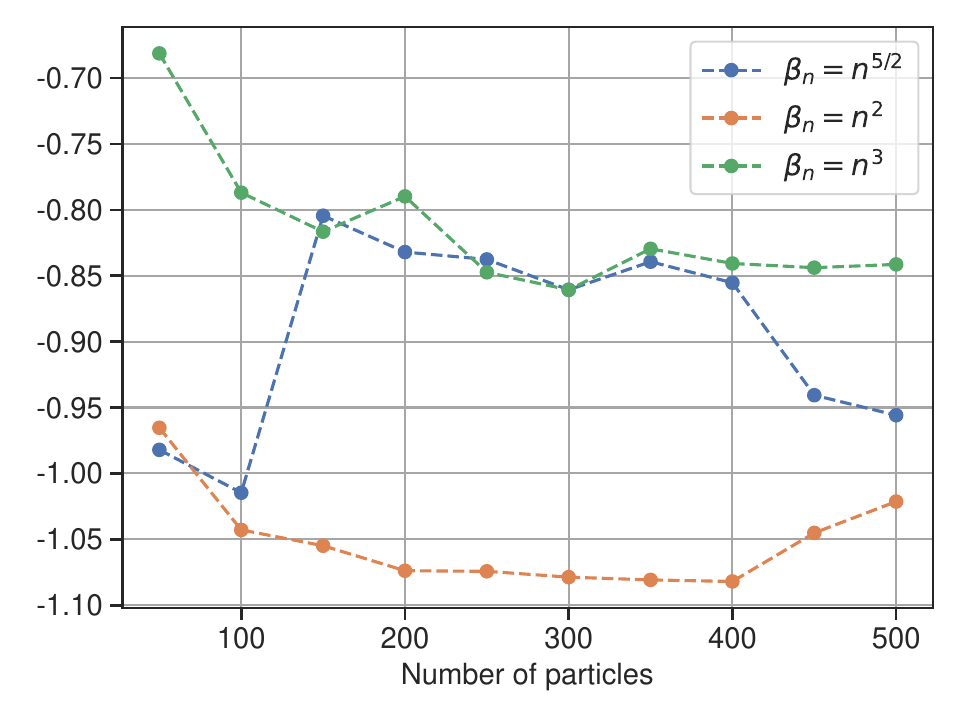}
    }
    \subfloat[Reweighted Coulomb background \label{f:comparison_energies_gibbs_temperature_coulomb}]{
        \includegraphics[width=\twofig]{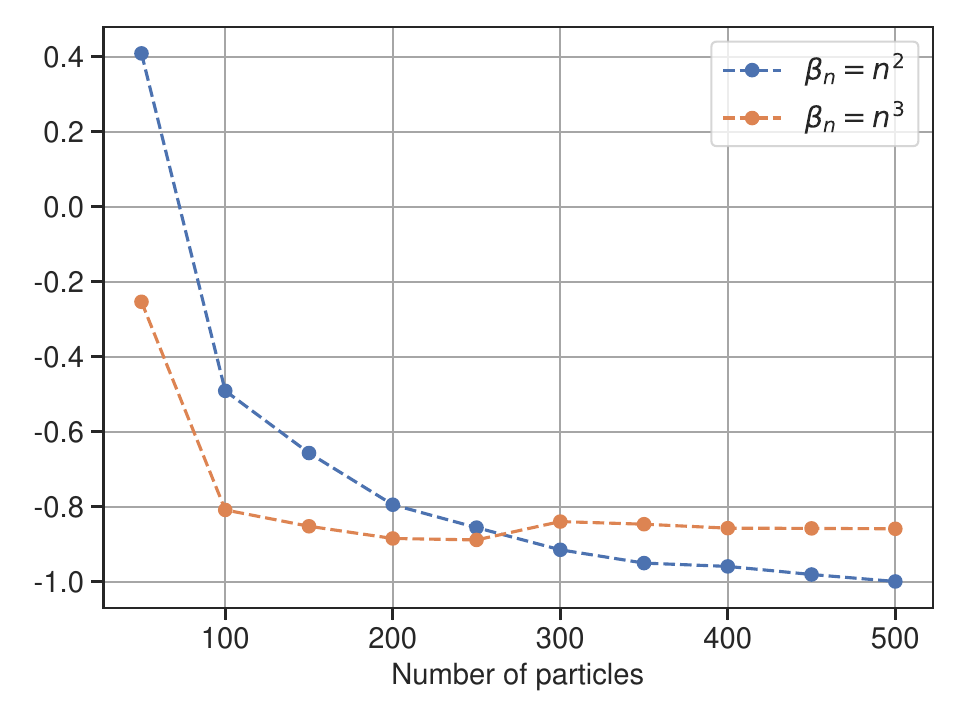}
    }
    \caption{Quantile of order $90\%$ of $I_{K}(\mu_n - \pi) - I_{K}(\pi)$ for different temperature schedules for the quenched Gibbs measure.}
    \label{f:comparison_energies_gibbs_temperature}
\end{figure}

We now look at the variance of linear statistics with respect to $\mu_n$ for the three different methods (i-iii) listed above. We conjecture that Gibbs measures can significantly improve on MCMC in that criterion, although we defer the corresponding theoretical investigation to further work.
We consider a simple test function in $\mathcal{H}_{K_{s, \epsilon}}$ given by $f(x) = \sum_{i = 1}^{10} a_i K_{s, \epsilon}(x, z_i)$ where the $(z_i)$ are chosen uniformly at random on $[-1, 1]^{d}$ and the $a_i$ are chosen uniformly at random on $[-1, 1]$.
We compute in \cref{f:variance} the empirical estimate of the variance $\mathrm{Var}\left[\frac{1}{n}\sum_{i = 1}^{n} f(y_i)\right]$ over $100$ independent realizations of $y_1, \dots, y_n$ for each method and each value of $n$, with MCMC background of size $M_n = 1\,000$ and $T = 10\,000$ MALA iterations for the quenched Gibbs measure.

\begin{figure}
    \centering
    \includegraphics[width = 0.8\textwidth]{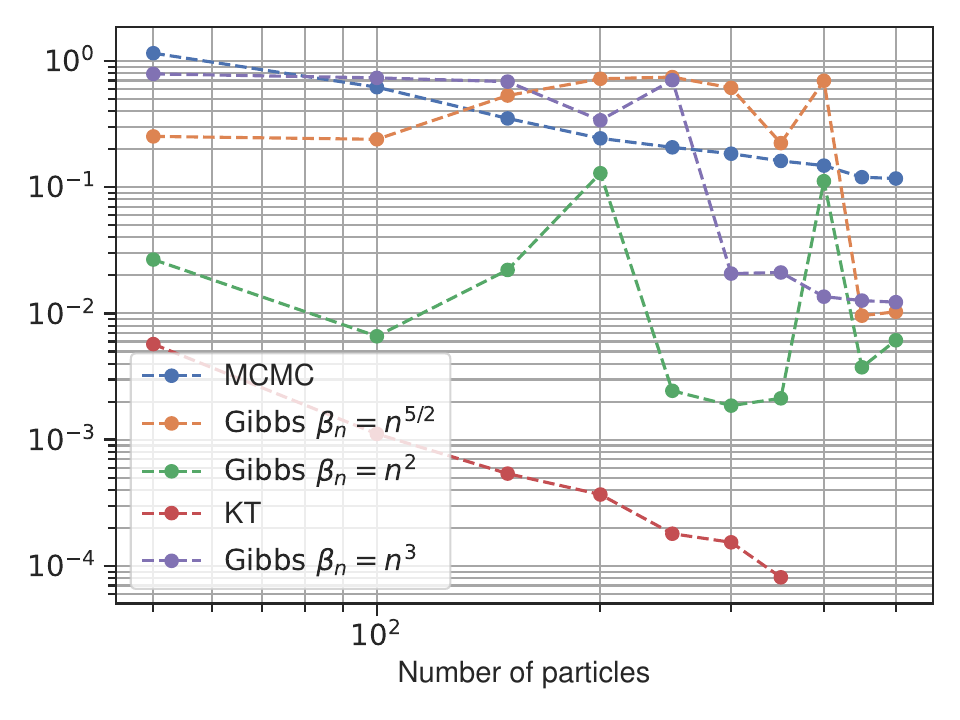}
    \caption{Empirical variance of $\frac{1}{n}\sum_{i = 1}^{n} f(x_i)$ in log scale for points obtained via MCMC, the quenched Gibbs measure with different temperature schedules, and kernel thinning.}
    \label{f:variance}
\end{figure}

While kernel thinning clearly exhibits a faster rate of convergence than MCMC, and the variance of the Gibbs estimator can --when tuned well-- fall between that of MCMC and kernel thinning, it is hard to conclude anything for the rates corresponding to Gibbs measures.
This is again related to the difficulty of tuning the MALA sampler as the Gibbs target changes.

\subsection{Bayesian classification}

In this section, we give a practical statistical illustration that a quenched Gibbs quadrature provides better coverage of simultaneous confidence intervals for a given computational budget, compared to classical MCMC.

First, we define the target integration task starting from the foundational expected utility principle; this is classical in Bayesian statistics, we include it for completeness.
Consider a training data set $(z_j, t_j)_{j = 1}^{N}$ where $z_j \in \mathcal{Z} = \mathbb{R}^{d}$ encodes the features of the j-th individual and $t_j \in \mathcal{T} = \{1, \dots, C\}$ is the associated class label.
An \emph{action} is the choice of a classifier $g : \mathcal{Z}^{N} \times \mathcal{T}^{N} \times \mathcal{Z} \rightarrow \mathcal{T}$.
For a new point $(z, t)$ with unknown class $t$, if one considers a joint probability measure $p$ on $z_1, t_1, \dots, z_N, t_N, z, t$ and the so-called 0 -- 1 utility function, the expected utility of the action $g$ is given by
\begin{equation}
  \label{e:seu}
  -\int 1_{t\neq g(z_{1:N}, t_{1:N}, z)} \mathrm{d} p(z_{1:N}, t_{1:N}, z, t),
\end{equation}
where $z_{1:N}$ is short for $(z_1, \dots, z_N)$.
The Bayes action or classifier is the function $g^\star$ which maximizes the expected utility~\eqref{e:seu}.
It is easy to see that it is given by
\begin{equation}
  \label{e:bayes_action}
  g^\star: z_{1:N}, t_{1:N}, z \mapsto \arg\max_{k=1}^C p(t=k \vert z_{1:N}, t_{1:N}, z).
\end{equation}
In other words, the Bayes classifier predicts the class $k$ with the highest conditional probability under $p$.
Making the classical assumption that $p$ factorizes as
\begin{align}
    \label{e:factorization}
  p(z_{1:N}, t_{1:N}, z, t) = \int p(y) \cdot \prod_{j=1}^N p(t_j \vert z_j, y) p(z_j) \cdot p(t\vert z, y) p(z) \mathrm{d}y,
\end{align}
where we introduced a parameter $y \in\mathcal{Y}$ and a reference measure $\mathrm{d}y$, the conditional in~\eqref{e:bayes_action} writes
\begin{align}
    \label{e:posterior_predictive}
    p(t = k |z_{1:N}, t_{1:N}, z) = \int p(t = k | z, y) \pi(y) \mathrm{d}y,
\end{align}
where
\[
  \pi(y) \propto \prod_{j = 1}^{N} p(t_j | z_j, y) p(y)
\]
is the density of the \emph{posterior} probability measure on the parameter $y$.
Using a Monte Carlo method to replace $\pi$ by an approximation $\mu_n = \frac{1}{n} \sum_{i = 1}^{n} \delta_{y_i}$, we approximate the Bayes action by
\begin{equation}
  \label{e:mc_action}
  g^{\mathrm{MC}}_n: z_{1:N}, t_{1:N}, z \mapsto \arg\max_{k=1}^C \int p(t=k \vert z, y) \mathrm{d}\mu_n(y).
\end{equation}

One way to guarantee that the expected utility~\eqref{e:seu} of $g^{\mathrm{MC}}_n$ is close to the one of $g^{\star}$ is to ensure that the $C$ integrals in the argmax of~\eqref{e:mc_action} simultaneously approximate the $C$ integrals in the argmax of~\eqref{e:bayes_action}.
Furthermore, if the dataset on which we want to perform classification contains $M$ test points instead of only one and we want to ensure that we approximate all predictions of the Bayes classifier using a single Monte Carlo run $y_1, \dots, y_n$, then there are $M\times C$ integrals to simultaneously approximate.
In that case, Theorem (2.7) and Theorem (2.10) show that the asymptotic confidence intervals obtained with a single run $y_{1:n}^{\mathrm{Q}}$ from the quenched Gibbs measure are tighter than the ones obtained for a single run $y_{1:n}^{\mathrm{MCMC}}$ from MCMC, depending on which smoothness one wants to assume on $p(t = k |z, y)$.
We empirically illustrate in the following that it is also the case at finite $n$ and for a MCMC-approximated sample of the quenched Gibbs measure, using proposal \eqref{e:mala_updates}.

To specify our setting, we consider binary logistic regression, where $C = 2$, $\mathcal{Z} = \mathbb{R}^{2}$, $\mathcal{T} = \{0, 1\}$, $ d =3$, and we assume $p$ to factorize as in~\eqref{e:factorization} with
\begin{equation}
    \label{e:logistic_regression}
    p(t = 1 | z, y) = 1 - p(t = 0 | z, y) \propto \exp\left(y^{\top}[z, 1]\right),
\end{equation}
where $[z, 1]$ stands for $x$ concatenated with $1$ to include an intercept.
To keep the posterior $\pi$ compactly-supported, we consider the prior $p(y)$ to be a three-dimensional Gaussian distribution with covariance $\sigma^2 I_3$ truncated at $\lvert x \rvert > 5 \sigma$, with $\sigma = 0.5$ as before, so that $\pi$ is compactly supported.
We mimick a toy dataset from the documention of the blackjax\footnote{\url{https://blackjax-devs.github.io/sampling-book/models/logistic_regression.html}} library: the features $z_{1:N}$ are drawn from a balanced mixture of two isotropic Gaussian distributions, and the class $t_{1:N}$ corresponds to the component membership.
We consider $N = 50$ individuals for the training set, and $M = 10$ test individuals to classify.
Computing the Bayes classifier for these $10$ individuals then requires to compute $M (C-1) = 10$ integrals with respect to $\pi$.
In order to be able to compare with a reference value for the Bayes classifier, we first approximate the conditional $p(t = 1 | z_{1:N}, t_{1:N}, z)$ in \eqref{e:posterior_predictive}, for each $z$ in the test set, using a long Metropolis--Hastings chain targeting $\pi$, with $10^5$ iterations.
Then, we compare the results to the Monte Carlo approximation 
\begin{equation}
    \label{e:bayes_action_mc}
    \hat{p}_{n}(t = 1 | z_{1:N}, t_{1:N}, z) = \frac{1}{n} \sum_{i = 1}^{n} p(t = 1 | z, y_i),
\end{equation}
when $y_1, \dots, y_n$ is i) the history of a Metropolis--Hastings chain targeting $\pi$ after $500$ burn-in iterations, 
ii) an approximate sample from the quenched Gibbs measure, after $T = 10\,000$ MALA iterations, with regularized Riesz kernel with $s = d -2$ and $\epsilon = 0.1$, at inverse temperature $\beta_n = n^2$, and with background $\nu_n$ of size $M_n = 1000$ given by an iid subsample of the history of a Metropolis--Hastings chain.
We build $50$ independent samples for each method, and compute for each value of $\delta$ the proportions $\mathfrak{p}_{\delta, \mathrm{MCMC}}$ and $\mathfrak{p}_{\delta, \mathrm{Gibbs}}$ of realizations of samples for which the $M$ estimators~\eqref{e:bayes_action_mc} simultaneously fall in the $M$ confidence regions
\[[\hat{p}_{n}(t = 1 | z_{1:N}, t_{1:N}, z) - \delta\,;\, \hat{p}_{n}(t = 1 | z_{1:N}, t_{1:N}, z) + \delta]\]
corresponding to the $M$ values $z$ in the test set.
We plot $\mathfrak{p}_{\delta, \mathrm{MCMC}}$ and $\mathfrak{p}_{\delta, \mathrm{Gibbs}}$ for different values of $\delta$ in \cref{f:proportions_bayesian_classification} when $n = 100$ for ii) and for different values of $n$ for i).
We use Gaussian confidence intervals at confidence level $95\%$ with Bonferroni correction for those $5$ values of $\delta$.

As expected, the Gibbs measure yields tighter simultaneous confidence intervals. In particular, $100$ particles drawn from the quenched Gibbs measure exhibit a coverage quality that is at least the same the history of a Metropolis--Hastings chain of size $10\,000$.

\section{Proof of auxiliary lemmas}
\label{s:auxiliary_proofs}
\subsection{Preliminaries}
\label{s:preliminaries}
We will repeatedly use a particular integral representation for $g$ and $K_{\zeta}$.
First, a quick computation shows that for any $c > 0$ and $\alpha > 0$,
\[\Gamma(\alpha) := \int_{0}^{+\infty} t^{\alpha -1}e^{-t}\,\mathrm{d}t = c^{\alpha} \int_{0}^{+\infty} t^{\alpha -1}e^{-ct}\,\mathrm{d}t.\]
choosing $\alpha = \frac{d-2}{2}$, this yields
\begin{align}
    g(z, y) &= \frac{1}{\Gamma\left(\frac{d-2}{2}\right)}\int_{0}^{+\infty} t^{\frac{d-2}{2}-1}\exp\left(- t \lvert z - y\rvert^2\right)\,\mathrm{d}t,\label{e:integral_representation_coulomb}\\
    K_{\zeta}(z, y) &= \frac{1}{\Gamma\left(\frac{d-2}{2}\right)}\int_{0}^{+\infty}t^{\frac{d-2}{2}-1}\exp\left(-n^{-2\zeta} t\right)\exp\left(- t \lvert z - y\rvert^2\right)\,\mathrm{d}t.\label{e:integral_representation_bounded}
\end{align}

It will be often convenient to work in Fourier space to bound the energies.
In fact \citep[Corollary 4]{sriperumbudur_hilbert_2009} states that for bounded translation invariant kernels $K(x, y) = \Psi(\lVert  x - y\rVert)$, 
the energy $I_g$ of a signed measure can be seen as the $L^{2}$-norm of the characteristic function of that measure, w.r.t. to a spectral measure which $\Psi$ is the Fourier transform of.
We give explicit expressions for this spectral measure in the following lemma for $K_{\zeta}$ and $g$, which proves the formula at the same time even though $g$ is not bounded.

\begin{lemma}
\label{l:energies_fourier}
Set \[\Lambda_c(\omega) := \frac{1}{(4\pi)^{d/2} \Gamma\left(\frac{d-2}{2}\right)}\int_{0}^{+\infty} t^{-2} \exp\left(-\frac{1}{4t}\lvert \omega\rvert^2)\right)\,\mathrm{d}t = C'\frac{1}{\lvert \omega \rvert^2},\]
where $C' >0$ is a constant, and \[\Lambda_{\zeta}(\omega) :=  \frac{1}{(4\pi)^{d/2}\Gamma\left(\frac{d-2}{2}\right)}\int_{0}^{+\infty} t^{-2}\exp(-n^{-2\zeta}t)\exp\left(-\frac{1}{4t}\lvert \omega\rvert^2)\right)\,\mathrm{d}t \leq \Lambda_{c}(\omega).\]
Then for any signed measure $\mu$ with finite energy $I_{K_\zeta}$, we have
\begin{align}
    K_{\zeta}(z, y) = \int e^{- i\langle z - y\,,\,\omega\rangle}\Lambda_{\zeta}(\omega)\,\mathrm{d}\omega,\\
    U_{K_{\zeta}}^{\mu}(z) = \int e^{- i\langle z\,,\,\omega\rangle}\Phi_{\mu}(\omega)\Lambda_{\zeta}(\omega)\,\mathrm{d}\omega,\\
    I_{K_{\zeta}}(\mu) = \int \left\lvert \Phi_{\mu}\right\rvert^{2}(\omega) \Lambda_{\zeta}(\omega)\,\mathrm{d}\omega,
\end{align}
where $\Phi_{\mu}$ denotes the characteristic function of $\mu$.
Similarly, if $\mu$ is a signed measure such that $I_{g}(\mu) < \infty$,
\begin{align}
    g(z, y) = \int e^{- i\langle z - y\,,\,\omega\rangle}\Lambda_{c}(\omega)\,\mathrm{d}\omega,\\
    U_{g}^{\mu}(z) = \int e^{- i\langle z\,,\,\omega\rangle}\Phi_{\mu}(\omega)\Lambda_{c}(\omega)\,\mathrm{d}\omega,\\
    I_{g}(\mu) = \int \left\lvert \Phi_{\mu}\right\rvert^{2}(\omega) \Lambda_c(\omega)\,\mathrm{d}\omega.
\end{align}
\end{lemma}
\begin{proof}
    Since the Gaussian is its own Fourier transform, we can write
    \begin{align}
    \label{e:fourier_gaussian}
        \exp\left(-t \lvert x - y \rvert^2\right) = \frac{1}{\left(4\pi t\right)^{d/2}}\int_{\mathbb{R}^{d}}\exp\left(i\langle x - y\,,\,\omega \rangle\right)\exp\left(-\frac{1}{4t}\lvert \omega\rvert^2\right)\,\mathrm{d}\omega.
    \end{align}
    Using Fubini and plugging \cref{e:fourier_gaussian} into Equations~\eqref{e:integral_representation_bounded} and~\eqref{e:integral_representation_coulomb}, we directly obtain the expressions for $U_{K_{\zeta}}^{\mu}$ and $U_{g}^{\mu}$.
    Next, using the fact that for any kernel $K$, $I_{K}(\mu) = \int U_{K}^{\mu}(z)\,\mathrm{d}\mu(z)$ and using Fubini again,
    we get the expressions for $I_{K_{\zeta}}(\mu)$ and $I_{g}(\mu)$.
\end{proof}

We give estimates on $\Lambda_{\zeta}$ in the next lemma, which we will use in the proof of \cref{p:max_pot_weight}.
\begin{lemma}
\label{l:spectral_measures}
\begin{align}
    &\int \Lambda_{\zeta}(\omega)\,\mathrm{d}\omega = n^{\zeta(d-2)},\label{e:mass_lambda_zeta}\\
    &\int_{\lvert \omega \rvert > n^{\alpha \zeta}} \Lambda_{\zeta}(\omega)\,\mathrm{d}\omega \leq C \exp\left(-\frac{1}{4}n^{\zeta(\alpha -1)}\right) \left( n^{\zeta(\alpha +1)(d/2 -1)}+n^{\zeta(d-2)}\right) \label{e:tail_lambda_zeta}
\end{align}
for any $\alpha > 1$.
\end{lemma}

\begin{proof}
    \cref{e:mass_lambda_zeta} is easy to see upon noting that $K_{\zeta}(z, z) = n^{\zeta(d-2)}$ is the Fourier transform of $\Lambda_{\zeta}$ at $0$ i.e. $\int \Lambda_{\zeta}(\omega)\,\mathrm{d}\omega$.

    We turn to \cref{e:tail_lambda_zeta}, we have
    \begin{align*}
        \int_{\lvert \omega \rvert > n^{\alpha \zeta}} \Lambda_{\zeta}(\omega)\,\mathrm{d}\omega &= C \int_{0}^{+\infty} t^{-2}e^{-n^{-2\zeta}t} \int \textbf{1}_{\lvert \omega \rvert > n^{\alpha \zeta}} \exp\left(-\frac{1}{4t}\lvert \omega\rvert^2\right)\,\mathrm{d}\omega\,\mathrm{d}t\\
        &\leq C \int t^{d/2 -2} e^{-t n^{-2\zeta}} e^{-\frac{1}{4t}n^{2\alpha \zeta}}\,\mathrm{d}t.
    \end{align*}
We split the integral in two parts, namely $\{t > n^{\zeta(\alpha + 1)}\}$ and $\{t < n^{\zeta(\alpha + 1)}\}$.
In the first case $t n^{-2\zeta} > \frac{1}{4t} n^{2\alpha\zeta}$ so we keep the first exponential term and bound the second one by $1$, and do the opposite in the second case.
We thus get
\begin{align*}
    \int_{\lvert \omega \rvert > n^{\alpha \zeta}} \Lambda_{\zeta}(\omega)\,\mathrm{d}\omega \leq C (I_1 + I_2)
\end{align*}
with \[I_1 = \int_{n^{\zeta(\alpha + 1)}}^{+\infty} t^{d/2 -2}e^{-t n^{-2\zeta}}\,\mathrm{d}t,\]
and \[I_2 = \int_{0}^{n^{\zeta(\alpha +1)}}t^{d/2 -2 }e^{-\frac{1}{4t}n^{2\alpha \zeta}}\,\mathrm{d}t.\]

We first bound $I_2$. Setting $t = n^{\zeta(\alpha +1)}$, we get
\begin{align*}
    I_2 &\leq n^{\zeta(\alpha+1)(d/2 -1)}\int_{0}^{1} u^{d/2-2} \exp\left(-\frac{1}{4u}n^{\zeta(\alpha -1)}\right)\,\mathrm{d}u\\
    &\leq n^{\zeta(\alpha +1)(d/2 -1)} \exp\left(-\frac{1}{4}n^{\zeta(\alpha -1)}\right) \int_{0}^{1} u^{d/2-2}\,\mathrm{d}u\\
    &\leq C n^{\zeta(\alpha +1)(d/2 -1)} \exp\left(-\frac{1}{4}n^{\zeta(\alpha -1)}\right).
\end{align*}
We now turn to $I_1$. Setting $t = u n^{2\zeta}$, we get
\begin{align*}
    I_1 &\leq n^{\zeta(d-2)}\int_{n^{\zeta(\alpha -1)}}^{+\infty} u^{d/2 -2}e^{-u}\,\mathrm{d}u.
\end{align*}
The second term is, up to a finite constant, the tail probability $\mathbb{P}\left( X > n^{\zeta(\alpha -1)}\right)$ when $X$ follows a Gamma distribution with shape parameter $d/2 -1$ and scale parameter $1$.
Recall that its Laplace transform is $\Psi_{X}(s) = \frac{1}{(1- s)^{d/2 -1}}$ when $s < 1$.
Using Markov's inequality, we get
\begin{align*}
    I_1 &\leq C n^{\zeta(d-2)}\mathbb{P}\left(e^{\frac{1}{2}X} > e^{\frac{1}{2}n^{\zeta(\alpha -1)}}\right)\\
    &\leq 2^{d/2 -1} C n^{\zeta(d-2)}\exp\left(-\frac{1}{2}n^{\zeta(\alpha -1)}\right).
\end{align*}
\end{proof}

We will also use the following boundedness property of Riesz potentials.
\begin{lemma}
    \label{l:riesz_potential_bound}
    Let $\mu$ be a probability measure with compact support and continuous density.
    Let $0 < \alpha \leq 2$, then the Riesz potential
    \[A_{\alpha}^{\mu}(z):= \int \frac{1}{\lvert y - z\rvert^{d-\alpha}}\mathrm{d}\mu(y)\]
    is uniformly bounded on $\mathbb{R}^{d}$, i.e. there exists a constant $C_{\alpha}$ such that, for all $z$, $A_{\alpha}^{\mu}(z) \leq C_{\alpha}$.
\end{lemma}

\begin{proof}
    We first invoke the maximum principle from \citep[Theorem 1.10, page 71]{landkof_foundations} stating that under the assumption $0 < \alpha \leq 2$, if $A_{\alpha}^{\mu}(z) \leq M$ holds for any $z$ in the support of $\mu$ then this holds for all $z \in \mathbb{R}^{d}$.
    Under the assumption on $\mu$ (i.e compact support and continuous density), $z \mapsto A_{\alpha}^{\mu}(z)$ is continuous \citep[Lemma 4.3]{chafai_first-order_2014} so that
    $A_{\alpha}^{\mu}$ is bounded on the support of $\mu$.
    Hence, using the maximum principle, there exists a constant $C_{\alpha} > 0$ such that $A_{\alpha}^{\mu}(z) \leq C_{\alpha}$ for all $z$.
\end{proof}

\subsection{Auxiliary lemmas for \cref{p:max_pot}}
\label{s:proof_max_pot_auxiliary}

\begin{proof}[Proof of \cref{l:comparison_g_Hzeta}]
$K_{\zeta}$ is a bit too rough of a regularization to compare directly the two integral representations, so we introduce yet another auxiliary kernel,
\[
    K_{\Lambda}^{s}(z, y) := \frac{1}{\Gamma\left(\frac{d-2}{2}\right)}\int_{0}^{+\infty}t^{\frac{d-2}{2}-1}\exp\left(-n^{-2\Lambda} t^{s}\right)\exp\left(- t \lvert z - y\rvert^2\right)
\]
with $\Lambda > 0$ and $0 < s < 1$.
Note that $K_{\zeta}$ can be recovered as a limiting case, choosing $\Lambda = \zeta$ and $s = 1$. 
We first compare $K_{\Lambda}^{s}$ to $g$ and then $K_{\zeta}$ to $K_{\Lambda}^{s}$.
\begin{align*}
    \Big\lvert \int g(z, y)\,\mathrm{d}\mu(y)& - \int K_{\Lambda}^{s}(z, y)\,\mathrm{d}\mu(y)\Big\rvert\\
    &= \frac{1}{\Gamma\left(\frac{d-2}{2}\right)}\Big\lvert \iint_{0}^{+\infty}\left(1 - \exp\left(-n^{-2\Lambda}t^s\right)\right)t^{\frac{d-2}{2}-1} \exp\left(-t \lvert z - y\rvert^2\right)\,\mathrm{d}t\,\mathrm{d}\mu(y)\Big\rvert\\
    &\leq \frac{n^{-2\Lambda}}{\Gamma\left(\frac{d-2}{2}\right)} \iint_{0}^{+\infty} t^{s} t^{\frac{d-2}{2}-1} \exp\left(-t \lvert z - y\rvert^2\right)\,\mathrm{d}t\,\mathrm{d}\mu(y)\\
    &= C n^{-2\Lambda} A_{2(1-s)}^{\mu}(z),
\end{align*}
with 
\[
    A_{2(1-s)}^{\mu}(z) =  \int\frac{1}{\lvert z - y\rvert^{d-2 + 2s}}\,\mathrm{d}\mu(y),
\]
which is the Riesz potential of $\mu$ introduced in \cref{l:riesz_potential_bound} with Riesz parameter $\alpha = 2(1-s)$.
Here we see that it would have been indeed insufficient to take $s = 1$, since we would have lost integrability at $0$.
Using \cref{l:riesz_potential_bound}, we thus get
\begin{align}
\label{e:unif_bound_g_hs}
    \underset{z \in \mathbb{R}^{d}}{\sup}\,\left\lvert \int g(z, y)\,\mathrm{d}\mu(y) - \int K_{\Lambda}^{s}(z, y)\,\mathrm{d}\mu(y)\right\rvert = \mathcal{O}\left(n^{-2\Lambda}\right).
\end{align}

In the same way, we have
\begin{align}
    \Big\lvert \int &K_{\zeta}(z, y)\,\mathrm{d}\mu(y) - \int K_{\Lambda}^{s}(z, y)\,\mathrm{d}\mu(y)\Big\vert\nonumber\\
    &\leq \frac{1}{\Gamma\left(\frac{d-2}{2}\right)}\iint_{0}^{+\infty} \left\lvert \exp\left(-n^{-2\zeta}t\right) - \exp\left(-n^{-2\Lambda}t^s\right)\right\rvert t^{\frac{d-2}{2}-1} \exp\left(-t \lvert z - y\rvert^2\right)\,\mathrm{d}t\,\mathrm{d}\mu(y). \label{e:intermediate_step}
\end{align}
We see that the term inside the absolute value is nonnegative iff $t^{1-s} \leq n^{2(\zeta - \Lambda)}$ and we thus split the integral in two parts.
First,
\begin{align*}
    \iint_{0}^{n^{\frac{2}{1-s}(\zeta - \Lambda)}} &\exp\left(-n^{-2\zeta}t\right)\left(1-\exp\left(-(n^{-2\Lambda}t^s - n^{-2\zeta}t)\right)\right)t^{\frac{d-2}{2}-1}\exp\left(-t\lvert z - y\rvert^2\right)\,\mathrm{d}t\,\mathrm{d}\mu(y)\\
    &\leq n^{-2\Lambda} \int \int_{0}^{n^{\frac{2}{1-s}(\zeta - \Lambda)}}\exp\left(-n^{-2\zeta}t\right)t^{\frac{d-2}{2}+s-1}\exp\left(-t\lvert z -y\rvert^2\right)\,\mathrm{d}t\,\mathrm{d}\mu(y)\\
    &\leq Cn^{-2\Lambda} A_{2(1-s)}^{\mu}(z)\\
    &= \mathcal{O}(n^{-2\Lambda})
\end{align*}
uniformly in $z$, where we bounded $\exp\left(-n^{-2\zeta}t\right) \leq 1$ and the integral over $t$ by the integral over $\mathbb{R}_{+}$, and we used \cref{l:riesz_potential_bound} again. 
The remaining part of \eqref{e:intermediate_step} is
\begin{align*}
    \iint_{n^{\frac{2}{1-s}(\zeta - \Lambda)}}^{+\infty} &\exp\left(-n^{-2\Lambda}t^s\right)\left(1-\exp\left(-(n^{-2\zeta}t - n^{-2\Lambda}t^s)\right)\right)t^{\frac{d-2}{2}-1}\exp\left(-t\lvert z - y\rvert^2\right)\,\mathrm{d}t\,\mathrm{d}\mu(y)\\
    &\leq n^{-2\zeta} \int \int_{n^{\frac{2}{1-s}(\zeta - \Lambda)}}^{+\infty}\exp\left(-n^{-2\Lambda}t^s\right) t^{\frac{d-2}{2}}\exp\left(-t\lvert z - y\rvert^2\right)\,\mathrm{d}t\,\mathrm{d}\mu(y).
\end{align*}
We cannot rely on the same argument again since we do not have integrability in the associated Riesz potential anymore. 
However, we bound $\exp\left(-t\lvert z -y \rvert^2\right) \leq 1$ and we perform the change of variable $h = t^s n^{-2\Lambda}$ i.e. $t = h^{1/s}n^{2\Lambda /s}$, with Jacobian $h^{1/s -1}n^{2\Lambda/s}$.
Bound the integral in $t$ by the integral over $\mathbb{R}_{+}$, we obtain
\[
    n^{-2\zeta + \frac{2\Lambda}{s}+ \frac{(d-2)\Lambda}{s}}\int_{0}^{+\infty}\exp\left(-h\right)h^{\frac{d-2}{2s}+ \frac{1}{s}-1}\,\mathrm{d}h = n^{-2\zeta + \frac{d\Lambda}{s}}\Gamma\left(\frac{d}{2s}\right).
\]
Choosing for instance $\Lambda = \frac{2s\zeta}{d+2s}$, we finally obtain
\begin{align}
\label{e:unif_bound_h_hs}
    \underset{z \in \mathbb{R}^{d}}{\sup}\,\left\lvert \int K_{\zeta}(z, y)\,\mathrm{d}\mu(y) - \int K_{\Lambda}^{s}(z, y)\,\mathrm{d}\mu(y)\right\vert = \mathcal{O}(n^{-4s\zeta/(d+2s)}).
\end{align}
Our choice of $\Lambda$ along with \cref{e:unif_bound_g_hs} and \cref{e:unif_bound_h_hs} then yields the result.
\end{proof}

\begin{proof}[Proof of \cref{l:taylor_no_weights}]
We start with
\begin{align*}
    \left\lvert U_{K_{\zeta}}^{\mu_n}(z) - U_{K_{\zeta}}^{\mu_{n}^{(\epsilon)}}(z)\right\rvert &\leq \frac{1}{n}\sum_{i = 1}^{n} \left\lvert K_{\zeta}(z, x_i) - \int K_{\zeta}(z, y)\,\mathrm{d}\lambda_{B(x_i, n^{-\epsilon})}(y)\right\rvert\\
    &= \frac{1}{n}\sum_{i = 1}^{n}\left\lvert \int \left\{ K_{\zeta}(z, x_i) - K_{\zeta}(z, x_i + n^{-\epsilon}u)\right\}\,\mathrm{d}\lambda_{B(0, 1)}(u)\right\rvert.
\end{align*}
We now Taylor expand $K_{\zeta}(z, .)$ around $x_i$ at order $2$.
Respectively denoting by $\nabla_2, \text{Hess}_2, \partial_{2}^{j, k}$ the gradient, Hessian and second order partial derivatives w.r.t the second coordinate of $K_{\zeta}$,
for each $z$ and $u$, there exists $c_u \in [0, 1]$ such that
\begin{align*}
    &\left\lvert \int \left\{ K_{\zeta}(z, x_i) - K_{\zeta}(z, x_i + n^{-\epsilon}u)\right\}\,\mathrm{d}\lambda_{B(0, 1)}(u)\right\rvert\\
    &= \left\lvert \int \left\{ n^{-\epsilon}\left\langle \nabla_{2} K_{\zeta}(z, x_i)\,,\, u\right\rangle + \frac{1}{2}n^{-2\epsilon}\left\langle \text{Hess}_2 K_{\zeta}(z, x_i + c_u n^{-\epsilon}u) u\,,\,u\right\rangle\right\}\,\mathrm{d}\lambda_{B(0, 1)}(u)\right\rvert.
\end{align*}
We easily see with the change of variable $u \mapsto -u$ that the first term is always zero.
Setting $x_u = x_i + c_u n^{-\epsilon} u$, it remains to bound the Hessian term
\[\left\lvert \sum_{1\leq j, k\leq d} \int \partial_{2}^{j, k} K_{\zeta}(z, x_u) u_j u_k \,\mathrm{d}\lambda_{B(0, 1)}(u)\right\rvert \leq \sum_{1 \leq j, k \leq d} \int \left\lvert \partial_{2}^{j, k} K_{\zeta}(z, x_u)\right\rvert \lvert u_{j} u_{k} \rvert\,\mathrm{d}\lambda_{B(0, 1)}(u).\]

A quick computation for the second order derivatives yields
\[\partial_{2}^{j, j} K_{\zeta}(z, y) = \left(\lvert z - y\rvert^{2} + n^{-2\zeta}\right)^{-d/2} (d-2) \left\{ d (z_j - y_j)^{2} \left(\lvert z - y\rvert^{2} + n^{-2\zeta}\right)^{-1} - 1\right\}\]
and
\[\partial_{2}^{j, k} K_{\zeta}(z, y) =  (d-2) d (z_j - y_j)(z_k - y_k) \left(\lvert z - y\rvert^{2} + n^{-2\zeta}\right)^{-1-d/2}\]
for $j \neq k$.
Bounding $\lvert z_j - y_j\rvert \lvert z_k - y_k\rvert \leq \max\left(\lvert z_j - y_j\rvert^2, \lvert z_k - y_k \rvert ^2\right)$,
we get that  $\lvert \partial_{2}^{j, k} K_{\zeta}(z, y)\rvert \leq (d-2)(d+1) n^{d \zeta}$ for all $j, k$.
Hence,
\[\left\lvert U_{K_{\zeta}}^{\mu_n}(z) - U_{K_{\zeta}}^{\mu_{n}^{(\epsilon)}}(z)\right\rvert \leq \frac{1}{2}n^{-2\epsilon + d \zeta}(d-2)(d+1) \sum_{1 \leq j , k\leq d}\int \lvert u_j u_k \rvert\,\mathrm{d}\lambda_{B(0, 1)}(u).\]
\end{proof}

\begin{proof}[Proof of \cref{l:bound_Ih_Ig}]
    This is straightforward using \cref{l:energies_fourier}.
\end{proof}

\begin{proof}[Proof of \cref{l:concentration}]
    
\cref{e:bound_Ig_concentration} is exactly Equation (3.3), page 17 of \citep{chafai_concentration_2018}.
\cref{e:bound_Zn_concentration} is an easy generalization of Lemma 4.1, page 20 of \citep{chafai_concentration_2018} with general $\beta_n$ and the exact same proof. 
Note that the $\frac{1}{2}$ factor difference in front of the first term only comes from a different choice of normalization in the energies.

For \cref{e:bound_Hn_concentration}, recall that
\[H_n(x_1, \dots, x_n) = \frac{1}{2n^2}\sum_{i \neq j} g(x_i, x_j) + \frac{1}{n}\sum_{i = 1}^{n} V(x_i)\]
and split the energy $H_n = \eta H_n + (1-\eta)H_n$.
Then, we lower bound the first term using $g \geq 0$; this term is kept to ensure integrability under $\mathbb{P}_{n, \beta_n}^{V}$.
The second term is lower bounded as follows. Use the superharmonicity of $g$ to write
\[
    \frac{1}{n^2}\sum_{i \neq j}g(x_i, x_j) \geq \frac{1}{n^2}\sum_{i, j}\iint g(x, y)\,\mathrm{d}\lambda_{B(x_i, \epsilon_n)}(x)\,\mathrm{d}\lambda_{B(x_j, \epsilon_n)}(y)-I_{g}(\lambda_{B(0, 1)})\frac{1}{n \epsilon_{n}^{d-2}}
\]
where the second term comes from the $n$ diagonal terms that are exceeding in the first term, and the fact that $I_{g}(\lambda_{B(0, n^{-\epsilon})}) = \frac{1}{\epsilon_{n}^{d-2}}I_{g}(\lambda_{B(0, 1)})$.
Then, we bound $V(x_i) - \int V\,\mathrm{d}\lambda_{B(x_i, \epsilon_n)}$ by Taylor expanding $V$ at order 2. 
This gives a uniform bound in $\mathcal{O}(\epsilon_{n}^{2})$ under the assumptions.
Grouping everything together and writing down the definition of $I_{g}(\mu_{n}^{(\epsilon)})$, we obtain \cref{e:bound_Hn_concentration}.
All of this is done in detail in \citep{chafai_concentration_2018}, Lemma 4.2 page 21 and steps 1 and 2 of the proof of Theorem 1.9 pages 21-23, and all the computations stay the same.

Finally, \cref{e:bound_concentration_intermediate} is obtained by combining the previous ones with our choice of parameter $\eta = c' n/\beta_n$. 
Note that in \citep{chafai_concentration_2018}, $\eta = n^{-1}$ since the authors work in the scaling $\beta_n = \beta n^2$.
Again, the computations stay the same (see step 3 of the proof of Theorem 1.9 of \citealp{chafai_concentration_2018}, pages 23-24) but we give a glimpse of it for the sake of completeness.
\begin{align*}
    \mathbb{P}_{n, \beta_n}^{V}(A) &= \frac{1}{Z_{n, \beta_n}^{V}}\int \exp\left(-\beta_n H_n(x_1, \dots, x_n)\right)\,\mathrm{d}x_1\,\dots\,\mathrm{d}x_n\\
    &\leq \frac{1}{Z_{n, \beta_n}^{V}}\exp\left(-(1-\eta)\beta_n\underset{A}{\inf}\, I_{g}^{V}(\mu_{n}^{(\epsilon)})+C\beta_n(1-\eta)\left(\frac{1}{n\epsilon_{n}^{d-2}} + \epsilon_{n}^{2}\right)\right)\\
    &\int \exp\left(-\beta_n \eta \frac{1}{n}\sum_{i = 1}^{n} V(x_i)\right)\,\mathrm{d}x_1\,\dots\,\mathrm{d}x_n.
\end{align*}
Due to the choice $\eta = c' n/\beta_n$, the remaining integral factorizes as $C^n$ with $C = \int \exp\left(-c' V\right) < +\infty$.
Using \cref{e:bound_Zn_concentration} then yields the result after splitting the leading order term of \cref{e:bound_Zn_concentration} into $(1-\eta)I_{g}^{V}(\mu_V)+\eta I_{g}^{V}(\mu_V)$, replacing $\eta$ by its expression and using \cref{e:bound_Ig_concentration}.
Note also that $\epsilon = 1/d$ in \citep{chafai_concentration_2018} but it is useful for us to keep freedom on this choice for now.
\end{proof}

\subsection{Auxiliary proofs for \cref{p:max_pot_weight}}

\begin{proof}[Proof of \cref{l:taylor_weights}]
    We begin as in \cref{l:taylor_no_weights}, writing
    \begin{align}
        \frac{1}{n}&\sum_{i = 1}^{n}(W \star \rho_{\eta})(x_i)\int K_{\zeta}(z, y)\,\mathrm{d}\lambda_{B(x_i, n^{-\epsilon})}(y) - \frac{1}{n}\sum_{i = 1}^{n}\int (W \star \rho_{\eta})(y) K_{\zeta}(z, y) \lambda_{B(x_i, n^{-\epsilon})}(y) \nonumber\\
        &= \frac{1}{n}\sum_{i = 1}^{n}\int \left\{ (W \star \rho_{\eta})(x_i) - (W\star \rho_{\eta})(x_i + n^{-\epsilon}u)\right\} K_{\zeta}(z, x_i + n^{-\epsilon}u)\,\mathrm{d}\lambda_{B(0, 1)}(u). \label{e:taylor_no_weights_rewright}
    \end{align}
A quick computation yields 
\begin{align*}
    \lvert\partial_{j} (W\star \rho_{\eta})(x)\rvert \leq \frac{1}{(2\pi\eta^2)^{d/2} \eta^{2}} \int W(y) \lvert x_j -y_j\rvert \exp\left(-\frac{1}{2\eta^2}\lvert x-y \rvert^2\right)\,\mathrm{d}y.
\end{align*}
Using the fact that $W$ is bounded and performing the change of variable $y \mapsto x - \eta u$, we get
\begin{align*}
    \lvert\partial_{j} (W\star \rho_{\eta})(x)\rvert \leq \lVert W\rVert_{\infty} \frac{1}{\eta} \frac{1}{(2\pi)^{d/2}}\int \lvert y\rvert \exp\left(-\lvert y\rvert^{2}/2\right)\,\mathrm{d}y.
\end{align*}
Taylor expanding $W \star \rho$ as in \cref{l:taylor_no_weights} at order one around $x_i$, and using the bound on the first-order derivative, yields
\begin{align*}
    &\left\lvert \frac{1}{n}\sum_{i = 1}^{n}\int \left\{ (W \star \rho_{\eta})(x_i) - (W\star \rho_{\eta})(x_i + n^{-\epsilon}u)\right\} K_{\zeta}(z, x_i + n^{-\epsilon}u)\,\mathrm{d}\lambda_{B(0, 1)}(u)\right\rvert \leq \frac{1}{\eta} \mathcal{O}\left(n^{-\epsilon + \zeta(d-2)}\right)
\end{align*}
where the $\mathcal{O}$ does not depend on $\eta$.
\cref{e:taylor_no_weights_rewright} then gives the result.
\end{proof}

\begin{proof}[Proof of \cref{l:fourier_weights}]
    Using \cref{l:energies_fourier}, we first write
    \begin{align*}
        \underset{z \in \mathbb{R}^{d}}{\sup}\,\left\lvert U_{K_{\zeta}}^{\nu_{n}^{\epsilon, \eta}}(z) - U_{K_{\zeta}}^{(W \star \rho_{\eta}).\mu_{V}'}(z)\right\rvert = \left\lvert \int e^{- i \langle z, \omega\rangle} \left(\Phi_{\nu_{n}^{\epsilon, \eta}}(\omega)-\Phi_{(W \star \rho_{\eta}).\mu_{V}'}(\omega)\right) \Lambda_{\zeta}(\omega)\,\mathrm{d}\omega \right\rvert
    \end{align*}
   Since $\nu_{n}^{\epsilon, \eta}$ has density $x \mapsto \frac{1}{n}\sum_{i = 1}^{n} \frac{1}{\lvert B(x_i, n^{-\epsilon})\rvert} (W\star \rho_{\eta})(x) \textbf{1}_{x \in B(x_i, n^{-\epsilon})}$, we have
    \[\left\lvert\Phi_{\nu_{n}^{\epsilon, \eta}} - \Phi_{(W \star \rho_{\eta}).\mu_{V}'}\right\rvert = \left\lvert \overline{\mathcal{F}(W\star \rho_{\eta})} \star \left(\Phi_{\mu_{n}^{(\epsilon)}}- \Phi_{\mu_V}\right)\right\rvert,\]
    so that
    \begin{align*}
        \underset{z \in \mathbb{R}^{d}}{\sup}\,\left\lvert U_{K_{\zeta}}^{\nu_{n}^{\epsilon, \eta}}(z) - U_{K_{\zeta}}^{(W \star \rho_{\eta}).\mu_{V}'}(z)\right\rvert \leq A_1 + A_2
    \end{align*}
    with 
    \begin{align*}
        A_1 &:= \int \int_{\lvert x \rvert <  n^{\zeta/s}} \left\lvert \mathcal{F}(W\star \rho_{\eta})(\omega -x)\right\rvert \left\lvert \Phi_{\mu_{n}^{(\epsilon)}}- \Phi_{\mu_V} \right\rvert(x) \,\mathrm{d}x\, \Lambda_{\zeta}(\omega)\,\mathrm{d}\omega,\\
        A_2 &:= \int \int_{\lvert x \rvert > n^{\zeta/s}} \left\lvert \mathcal{F}(W\star \rho_{\eta})(\omega -x)\right\rvert \left\lvert \Phi_{\mu_{n}^{(\epsilon)}}- \Phi_{\mu_V} \right\rvert(x) \,\mathrm{d}x\, \Lambda_{\zeta}(\omega)\,\mathrm{d}\omega.
    \end{align*}
    for some $0 < s < 1$.
    
    We first bound $A_2$.
    We bound $\left\lvert \Phi_{\mu_{n}^{(\epsilon)}}- \Phi_{\mu_V} \right\rvert(x) \leq 2$.
    Using 
    \[
        \lvert \mathcal{F}(W\star \rho_{\eta})(\omega - x)\rvert = \lvert\mathcal{F}(W)(\omega -x)\mathcal{F}(\rho_{\eta})(\omega - x)\rvert \leq C \lVert W\rVert_{L^{1}(\mathbb{R}^{d})} \exp(-\eta^2 \lvert \omega - x\rvert^2 /2),
    \] 
    we obtain 
    \begin{align*}
        A_2 &\leq 2 C \lVert W\rVert_{L^{1}(\mathbb{R}^{d})}\int \int_{\lvert x \rvert > n^{\zeta/s}} \exp\left(-\eta^2 \lvert \omega - x\rvert^2 /2\right)\,\mathrm{d}x\,\Lambda_{\zeta}(\omega)\,\mathrm{d}\omega.
    \end{align*}
    For $1 < \alpha < 1/s$, we split again the first integral in $\lvert \omega\rvert < n^{\alpha \zeta}$ and $\lvert \omega\rvert > n^{\alpha \zeta}$.
    When $\lvert \omega\rvert > n^{\alpha \zeta}$, we bound the exponential integral by the integral over the whole domain which is bounded by $\frac{C}{\eta^{d}}$.
    We then use the estimate of \cref{e:tail_lambda_zeta} to bound the remaining term.
    In the case $\lvert \omega \vert < n^{\alpha \zeta}$, we then have $\lvert \omega - x\rvert > n^{\zeta/s}-n^{\alpha \zeta}$ and we can use the decay of the Gaussian tail since $\alpha < 1/s$, along with \cref{e:mass_lambda_zeta}.
    Doing so, we get
    \begin{align*}
        &A_2  \leq \frac{C \lVert W\rVert_{L^{1}(\mathbb{R}^{d})}}{\eta^{d}}\int_{\lvert \omega\rvert > n^{\alpha \zeta}}\Lambda_{\zeta}(\omega)\,\mathrm{d}\omega \\
        &+ C \lVert W\rVert_{L^{1}(\mathbb{R}^{d})} \iint_{\lvert \omega - x\rvert > n^{\zeta/s}-n^{\alpha \zeta}} \exp\left(-\eta^2 \lvert \omega - x\rvert^2 /2\right)\,\mathrm{d}x\,\Lambda_{\zeta}(\omega)\,\mathrm{d}\omega\\
        &\leq \frac{C \lVert W\rVert_{L^{1}(\mathbb{R}^{d})}}{\eta^{d}} \left[e^{-\frac{1}{4}n^{\zeta(\alpha -1)}} \left( n^{\zeta(\alpha +1)(d/2 -1)}+n^{\zeta(d-2)}\right) + n^{\zeta(d-2)}\exp\left(-\eta^2 \left(n^{\zeta/s}-n^{\alpha \zeta}\right)^2 /2\right)\right].
    \end{align*}

    We now turn to bounding $A_1$.
    Recall that $\Phi_{\mu_{n}^{(\epsilon)}-\mu_V}(x) = \int e^{i \langle y\,,\, x \rangle}\,\mathrm{d}(\mu_{n}^{(\epsilon)}-\mu_V)(y)$.
    Since $y \mapsto e^{i \langle y\, , \, x\rangle}$ is $\lvert x\rvert$-Lipschitz and bounded by $1$, 
    \[\underset{\lvert x \rvert < n^{\zeta/s}}{\sup\,}\left\lvert \Phi_{\mu_{n}^{(\epsilon)}-\mu_V}(x)\right\rvert \leq n^{\zeta /s} \underset{\lVert f \rVert_{Lip} \leq 1,\, \lVert f\rVert_{\infty} \leq 1}{\sup\,}\left\lvert\int f\,\mathrm{d}(\mu_{n}^{(\epsilon)}-\mu_V)\right\rvert = n^{\zeta/s} d_{\text{BL}}(\mu_{n}^{(\epsilon)}, \mu_V).\]
    
    Moreover we have
    \begin{align*}
        \lVert \mathcal{F}(W\star\rho_{\eta})\rVert_{L^{1}(\mathbb{R}^{d})} &= \left\lVert \mathcal{F}(W)\mathcal{F}(\rho_{\eta})\right\rVert_{L^{1}(\mathbb{R}^{d})}\\
        & \leq \lVert \mathcal{F}(W)\rVert_{L^{2}(\mathbb{R}^{d})} \lVert \mathcal{F}(\rho_{\eta})\rVert_{L^{2}(\mathbb{R}^{d})}\\
        & = \lVert W\rVert_{L^{2}(\mathbb{R}^{d})} \lVert \rho_{\eta}\rVert_{L^{2}(\mathbb{R}^{d})}\\
        & = \frac{C}{\eta^{d/2}} \lVert W\rVert_{L^{2}(\mathbb{R}^{d})} 
    \end{align*}
    where we used Cauchy--Schwarz and Plancherel theorem.
    Thus, using \cref{e:mass_lambda_zeta},
    \begin{align*}
        A_1 &\leq \underset{\lvert x \rvert < n^{\zeta/s}}{\sup\,}\left\lvert \Phi_{\mu_{n}^{(\epsilon)}-\mu_V}(x)\right\rvert \lVert \mathcal{F}(W\star\rho_{\eta})\rVert_{L^{1}(\mathbb{R}^{d})} \int \Lambda_{\zeta}(\omega)\,\mathrm{d}\omega\\
        &\leq \frac{C}{\eta^{d/2}} \lVert W\rVert_{L^{2}(\mathbb{R}^{d})} n^{\zeta(d-2 + 1/s)} d_{\text{BL}}(\mu_{n}^{(\epsilon)}, \mu_V).
    \end{align*}
    which concludes the proof of the lemma.
\end{proof}

\end{document}